\documentclass[letterpaper]{article} %
\usepackage{aaai23}  %
\usepackage{times}  %
\usepackage{helvet}  %
\usepackage{courier}  %
\usepackage[hyphens]{url}  %
\usepackage{graphicx} %
\usepackage{natbib}  %
\usepackage{caption} %
\usepackage{algorithm}
\usepackage{algorithmic}

\usepackage{newfloat}
\usepackage{listings}

\usepackage{amsmath,amsfonts,bm,amssymb,mathtools,amsthm}
\usepackage{color,xcolor,xspace}
\usepackage{booktabs}
\usepackage{thm-restate}
\usepackage{bbding}
\usepackage{pifont}
\usepackage{wasysym}

\newtheorem{theorem}{Theorem}
\newtheorem{definition}{Definition}
\newtheorem{proposition}{Proposition}
\newtheorem{lemma}{Lemma}
\newtheorem{assumption}{Assumption}

\def\eqref#1{Eq.~(\ref{#1})}

\def\1{\bm{1}}

\newcommand{\defeq}{\vcentcolon=}

\DeclareMathAlphabet{\mathsfit}{\encodingdefault}{\sfdefault}{m}{sl}
\SetMathAlphabet{\mathsfit}{bold}{\encodingdefault}{\sfdefault}{bx}{n}

\def\gA{{\mathcal{A}}}
\def\gB{{\mathcal{B}}}

\def\gD{{\mathcal{D}}}

\def\gM{{\mathcal{M}}}

\def\gS{{\mathcal{S}}}
\def\gT{{\mathcal{T}}}

\def\gX{{\mathcal{X}}}

\def\sR{{\mathbb{R}}}

\newcommand{\rhat}{\hat{r}}
\newcommand{\hdagger}{(h^\dagger)}
\newcommand{\E}{\mathbb{E}}

\newcommand{\KL}{D_{\mathrm{KL}}}
\newcommand{\DF}{D_f}
\newcommand{\DFbeta}{D_{\widetilde{f}_\beta}}
\newcommand{\Fbeta}{\widetilde{f}_\beta}

\DeclareMathOperator*{\argmax}{arg\,max}
\DeclareMathOperator*{\argmin}{arg\,min}

\newcommand{\methodshort}{RelaxDICE}
\newcommand{\methodshortdrc}{RelaxDICE-DRC}

\DeclareCaptionStyle{ruled}{labelfont=normalfont,labelsep=colon,strut=off} %
\floatstyle{ruled}
\newfloat{listing}{tb}{lst}{}
\floatname{listing}{Listing}
\title{Offline Imitation Learning with Suboptimal Demonstrations via \\ Relaxed Distribution Matching}

\author{
    Lantao Yu\equalcontrib \textsuperscript{\rm 1},
    Tianhe Yu\equalcontrib \textsuperscript{\rm 1},
    Jiaming Song \textsuperscript{\rm 2},
    Willie Neiswanger \textsuperscript{\rm 1},
    Stefano Ermon \textsuperscript{\rm 1}
}
\affiliations{
    \textsuperscript{\rm 1}
    Computer Science Department, Stanford University\\
    \textsuperscript{\rm 2}
    NVIDIA (Work done while at Stanford)\\
    \texttt{\{lantaoyu,tianheyu,tsong,neiswanger,ermon\}@cs.stanford.edu}
}

\begin{document}

\maketitle

\begin{abstract}
Offline imitation learning (IL) promises the ability to learn performant policies from pre-collected demonstrations without interactions with the environment.
 However, imitating behaviors fully offline typically requires numerous expert data. To tackle this issue, we study the setting where we have limited expert data and supplementary suboptimal data. In this case, a well-known issue is the distribution shift between the learned policy and the behavior policy that collects the offline data.
 Prior works mitigate this issue by regularizing the KL divergence between the stationary state-action distributions of the learned policy and the behavior policy. We argue that such constraints based on exact distribution matching can be overly conservative 
  and hamper policy learning, especially when the imperfect offline data is highly suboptimal. To resolve this issue, we present RelaxDICE, which employs an asymmetrically-relaxed $f$-divergence for explicit support regularization. 
  Specifically, instead of driving the learned policy to exactly match the behavior policy, we impose little penalty whenever the density ratio between their stationary state-action distributions is upper bounded by a constant.
  Note that such formulation leads to a nested min-max optimization problem, which causes instability in practice. RelaxDICE addresses this challenge by supporting a closed-form solution for the inner maximization problem.
  Extensive empirical study shows that our method significantly outperforms the best prior offline IL method in six standard continuous control environments with over 30\% performance gain on average, across 22 settings where the imperfect dataset is highly suboptimal.
\end{abstract}

\section{Introduction}
Imitation learning (IL) \citep{pomerleau1988alvinn,GAIL2016Ho,Dagger2011Ross} studies the problem of programming agents directly with expert demonstrations. However, successful IL usually demands a large amount of optimal trajectories, and many adversarial IL methods~\citep{GAIL2016Ho,AIRLFu2018,ke2020imitation,kostrikov2018discriminator} require online interactions with the environment to get samples from intermediate policies for policy improvement. Considering these limitations,
we focus on the setting of offline imitation learning with supplementary imperfect demonstrations~\citep{kim2021demodice}, which holds the promise of addressing these challenges (i.e. no large collection of expert data and no online interactions with the environment during training). Specifically, we aim to learn a policy using a small amount of expert demonstrations and a large collection of trajectories with unknown level of optimality 
that are typically cheaper to obtain.

As in prior offline reinforcement learning (RL) and offline policy evaluation works, offline IL~\citep{kim2021demodice} also has the distribution shift problem~\citep{levine2020offline,kumar2019stabilizing,fujimoto2018off}: the agent performs poorly during evaluation because the learned policy deviates from the behavior policy used for collecting the offline data.
To mitigate this problem, prior works based on \emph{distribution correction estimation} (the ``DICE'' family) \citep{nachum2019dualdice,nachum2019algaedice,lee2021optidice,kim2021demodice,ValueDICE2019Kostrikov,zhang2020gendice,zhang2020gradientdice,yang2020off} collectively use a distribution divergence measure (e.g. $f$-divergence) to regularize the learned policy to be similar to the behavior policy. 
However, such regularization schemes based on exact distribution matching can be overly conservative. For example, in settings where the offline data is highly suboptimal, such an approach will require careful tuning of the regularization strength (denoted as $\alpha$) in order to find the delicate balance between policy optimization on limited expert data and policy regularization to the behavior policy. Otherwise, the resulting policy will either suffer from large distribution shift because of small $\alpha$ or behave too similarly to the suboptimal behavior policy due to large $\alpha$.
We argue that a more appropriate regularization for offline imitation learning with limited expert data and diverse supplementary data is indispensable, which is the goal of this work.

\begin{figure*}[t]
\centering
\includegraphics[width=.96\textwidth]{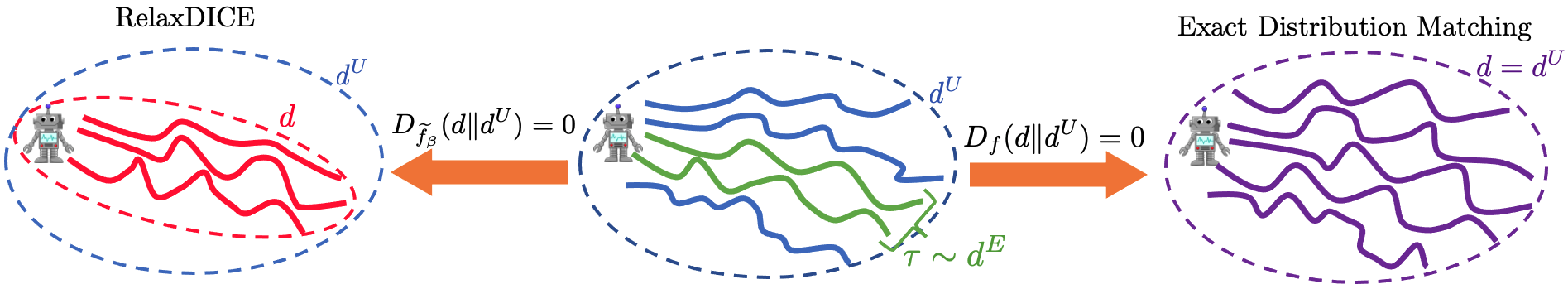}
\caption{\footnotesize Illustration of regularizations based on relaxed distribution alignment (left) and exact distribution matching (right). The curves represent trajectories sampled from the expert policy (green), the behavior policy that collects the suboptimal data (blue), and the learned policy (red and purple) under different kinds of regularization. Dashed lines represent the support of these distributions.}
\label{fig:illustration}
\end{figure*}

Towards this end, we draw inspiration from domain adaptation theory \citep{wu2019domain} and present \methodshort{}, which employs an asymmetrically-relaxed $f$-divergence for explicit support regularization instead of exact distribution matching between the learned policy and the suboptimal behavior policy.
On one hand, we still encourage the learned policy to stay within the support of the pre-collected dataset such that policy evaluation/improvement is stable and reliable. On the other hand, we will not drive the learned policy to exactly match the behavior policy since the offline demonstrations have unknown level of optimality (see Figure~\ref{fig:illustration} for illustration). 
Different from \citep{wu2019behavior,levine2020offline} which tried to directly regularize the policies and observed little benefits in the context of offline RL, we enforce such a regularization over stationary state-action distributions to effectively reflect the diversity in both states and actions (rather than enforce constraints only on policies/action distributions). However, this leads to a nested min-max optimization problem that causes instability during training.
We surprisingly found that our new formulation enjoys a closed-form solution for the inner maximization problem, thus preserving the key advantage of previous state-of-the-art DICE methods~\citep{lee2021optidice,kim2021demodice}.
Furthermore, the stationary state-action distribution of the suboptimal behavior policy can be potentially modified to be closer to that of the expert policy, by leveraging an approximate density ratio obtained from expert and suboptimal data. Thus we further propose \methodshortdrc{}, an extension of \methodshort{} by penalizing the relaxed $f$-divergence between the stationary state-action distributions of the learned policy and the density-ratio-corrected behavior policy.
This method also enjoys a desirable closed-form solution for the inner maximization and a potential for better policy improvement.

We empirically evaluate our method on a variety of continuous control tasks using environments and datasets from the offline RL benchmark D4RL~\citep{fu2020d4rl}. We construct datasets where there are a small amount of expert demonstrations and a large collection of imperfect demonstrations with different levels of optimality following the design choice in \cite{kim2021demodice}. More importantly, for each environment, we design up to four different settings that are much more challenging than the ones in \citep{kim2021demodice}, in the sense that the supplementary imperfect data are highly suboptimal. Extensive experimental results show that our method outperforms the most competitive prior offline IL method across all 22 tasks by an average margin over 30\%.
Furthermore, \methodshort{} is much more performant and robust with respect to hyperparameter changes than prior works~\citep{kim2021demodice} in our challenging settings,
demonstrating the superiority of our relaxed distribution matching scheme for offline imitation learning. 
\section{Background}\label{sec:background}
\paragraph{Markov Decision Process.} A Markov decision process (MDP) is defined by $\gM = \langle \gS, \gA, T, r, p_0, \gamma \rangle$, where $\gS$ is a set of states; $\gA$ is a set of actions; $T: \gS \times \gA \to \Delta(\gS)$ is the transition distribution and $T(s_{t+1}|s_t, a_t)$ specifies the probability of transitioning from state $s_t$ to state $s_{t+1}$ by executing action $a_t$; $p_0 \in \Delta(\gS)$ is the initial state distribution; $R: \gS \times \gA \to \sR$ is the reward function; and $\gamma \in [0,1]$ is the discount factor. A policy $\pi: \gS \to \Delta(\gA)$ maps from states to distributions over actions, which together with the MDP $\gM$, induces a stationary state-action distribution $d^\pi(s,a)$  (also called occupancy measure):
{\small
\begin{align*}
    d^\pi(s,a) = (1 - \gamma) \sum_{t=0}^\infty \gamma^t \mathrm{Pr}(& s_t = s, a_t=a|s_0 \sim p_0, \\
     & a_t \sim \pi(\cdot | s_t), s_{t+1} \sim T(\cdot | s_t, a_t)).
\end{align*}}
Here $1 - \gamma$ is a normalization factor such that the occupancy measure is a normalized distribution over $\gS \times \gA$. Because of the one-to-one correspondence described in the following theorem,
a policy optimization problem can be equivalently formulated as an occupancy measure optimization problem.

\begin{theorem}[\citep{feinberg2012handbook,syed2008apprenticeship}]\label{the:occupancy-and-policy} Suppose $d$ satisfies the following Bellman flow constraints:
{\small
\begin{align}
    & \sum_a d(s,a) = (1 - \gamma) p_0(s) + \gamma \sum_{s', a'} T(s|s',a') d(s', a'), \forall s. \nonumber \\
    & \text{and~~} d(s, a) \geq 0, \forall s, a.
    \label{eq:bellman-flow-constrains}
\end{align}
}
Define $\pi_d(a|s) \defeq \frac{d(s,a)}{\sum_{a'} d(s, a')}$. Then $d$ is the occupancy measure for $\pi_d$. Conversely if $\pi$ is a policy such that $d$ is its occupancy measure, then $\pi = \pi_d$ and $d$ satisfies \eqref{eq:bellman-flow-constrains}.
\end{theorem}

The Bellman flow constraints in \eqref{eq:bellman-flow-constrains} essentially characterize all possible occupancy measures consistent with the MDP, such that they can be induced by some policies. Therefore it is necessary to enforce these constraints when we design optimization problems over occupancy measures.

\paragraph{IL with Expert Data.} We can learn performant policies via imitation learning when a set of expert demonstrations $\gD^E$ is provided. The expert dataset $\gD^E = \{(s,a,s')\}$ is generated according to $(s,a) \sim d^E,~s' \sim T(\cdot|s,a)$, where $d^E$ is the occupancy measure of the expert policy. A classical IL approach is behavior cloning (BC), which optimizes a policy $\pi$ by minimizing the expected KL between $\pi^E(\cdot|s)$ and $\pi(\cdot|s)$ for $s \sim d^E(s)$ (the state marginal of expert occupancy measure):
{\small
\begin{align*}
    \argmin_{\pi} ~\E_{d^E(s)}\left[\KL\left(\pi^E(\cdot|s) \| \pi(\cdot|s)\right)\right]
    = &- \E_{d^E(s,a)} [\log \pi(a|s)].
\end{align*}
}
Alternatively, IL can be formulated as minimizing the $f$-divergence between occupancy measures: $\min_d \DF(d \| d^E)$ \citep{ho2016generative,ValueDICE2019Kostrikov,ke2020imitation,ghasemipour2020divergence}. However, since estimating and minimizing $f$-divergence requires the unknown density ratio $d/d^E$, which can be obtained only through variational estimation using samples from $d^E$ and $d$ (all intermediate policies), these IL methods are not offline and have to use adversarial training.

\paragraph{Offline IL with Expert and Non-Expert Data.} The standard IL setting above typically requires a large amount of optimal demonstrations from experts, and sometimes require online interactions with the MDP.
To address these limitations, researchers proposed to study offline IL with limited expert data and supplementary imperfect data \citep{kim2021demodice}, a meaningful yet challenging setting where no interaction with the environment is allowed, and we only have a small amount of expert demonstrations $\gD^E$ and an additional collection of suboptimal demonstrations $\gD^U$ with unknown level of optimality. The pre-collected dataset $\gD^U = \{(s,a,s')\}$ is generated according to $(s,a) \sim d^U,~s' \sim T(\cdot|s,a)$ with $d^U$ being the occupancy measure of some unknown behavior policy. In this setting, the key is to study how to leverage the additional imperfect dataset $\gD^U$ to provide proper regularization to help the policy/occupancy measure optimization on $\gD^E$. Towards this end, DemoDICE \citep{kim2021demodice} extends the offline RL method OptiDICE \citep{lee2021optidice} and uses $\KL(d \| d^U)$ to realize the regularization. Moreover, we note that a key to their success is both OptiDICE and DemoDICE avoid the nested min-max optimization \citep{nachum2019algaedice} by supporting a closed-form solution for their inner maximization problem.

\paragraph{Density Ratio Estimation via Classification.} Thanks to the connection between density ratio estimation and classification \citep{menon2016linking,yu2021unified}, given samples from two distributions $p$ and $q$, we can use any strictly proper scoring rule and a link function $\psi_\mathrm{dr}$ to recover the density ratio $p/q$. For example, we can use logistic regression to approximately recover $d^E/d^U$:
\begin{align}
    c^* = \argmax_{c: \gS \times \gA \to (0,1)} & \E_{d^E(s,a)}[\log c(s,a)] + \nonumber\\
    & \E_{d^U(s,a)} [\log (1 - c(s,a))]
    \label{eq:binary-classification}
\end{align}
Since $c^*(s,a) = \frac{d^E(s,a)}{d^E(s,a) + d^U(s,a)}$, the optimal density ratio can be recovered as:
\begin{align}
    r^*(s,a) = \psi_\mathrm{dr}(c^*(s,a)) = \frac{c^*(s,a)}{1 - c^*(s,a)} = \frac{d^E(s,a)}{d^U(s,a)}
    \label{eq:link-function}
\end{align}

\section{Offline IL with Suboptimal
Demonstrations via \methodshort{}}\label{sec:method}
In this section, we present \methodshort{}, a novel method for offline imitation learning with expert and supplementary non-expert demonstrations. A key question to study in this meaningful yet challenging setting is how to derive offline algorithms with appropriate regularization $\Omega(d, d^U)$ to effectively leverage the additional imperfect dataset $\gD^U$. Formally, we begin with the following constrained optimization problem over the occupancy measure:
\begin{align}
    & \max_{d \geq 0}~ - \KL(d \| d^E) - \alpha \Omega(d, d^U) \label{eq:objective-constrained-optimization}\\
    & \mathrm{s.t.}~~ \sum_a d(s,a) = (1 - \gamma) p_0(s) + \nonumber\\
    &~~~~~~~~~~~~~~~~~~~~~~~~~~~~~~~~
    \gamma \sum_{s', a'} T(s|s',a') d(s', a'), \forall s \in \gS. \label{eq:bellman-flow-constrained-optimization}
\end{align}
where \eqref{eq:bellman-flow-constrained-optimization} is the Bellman flow constraints introduced in Theorem~\ref{the:occupancy-and-policy} that any valid occupancy measure must satisfy, and $\alpha > 0$ is a weight factor balancing between minimizing KL divergence with $d^E$ (estimated with the limited expert data) and preventing deviation from $d^U$. For example, a popular regularization choice in prior offline IL and offline RL works is the $f$-divergence $\DF(d \| d^E)$, which was originally designed for exact distribution matching between a model distribution and a target distribution \citep{nowozin2016f}. Although this choice can indeed enforce $d$ to be close to $d^U$, we think that divergences or distances for exact distribution matching can be overly conservative and may lead to undesired effects when $d^U$ is highly suboptimal. In this case, even the true optimal occupancy measure (corresponding to the true optimal policy) will incur a high penalty from $\Omega(d, d^U)$.
Although we can reduce $\alpha$ to mitigate the negative effect, we cannot remove the bias unless $\alpha$ approaches zero, which will then leave us at risk of exploring out-of-support state-actions because of a too small regularization strength. Moreover, prior theoretical work on offline RL \citep{zhan2022offline} also suggests that a smaller $\alpha$ will lead to a worse sample complexity and a higher error floor. 
Proofs for this section can be found in the appendix.

\subsection{An Optimistic Fix to the Pessimistic Regularization}
To ensure the suboptimal dataset contains useful information about the optimal policy $\pi^*$, theoretical studies typically make some assumptions about $\gD^U$. As a motivating example, a minimal assumption adopted in \citep{zhan2022offline} is the following $\pi^*$-concentrability\footnote{This assumption is much weaker than the all-policy concentrability in prior theoretical works \citep{munos2008finite,farahmand2010error,chen2019information}} (where $d^*$ is the occupancy measure of $\pi^*$):

\begin{assumption}\label{assumption:concentrability}
$d^U(s,a) > 0$ and there exists a constant $B$ such that
$d^{*}(s,a)/d^U(s,a) \leq B, \forall s, a.$
\end{assumption}

Under this assumption, we argue that an ideal regularization $\Omega(d, d^U)$ would aim to bound the density ratio $d / d^U$ by a constant, instead of driving towards $d \equiv d^U$.  In other words, we still want to regularize $d$ to stay in the support of $d^U$ so that policy evaluation/improvement is stable and reliable under a small distribution shift, but different from a divergence like $\DF(d \| d^U)$, we will impose little penalty on $d$ if $d / d^U \leq B$,
so that we will not  enforce $d$ to exactly match $d^U$ and the optimal policy can be preserved under the regularization (i.e., $d^* \in \argmin_d \Omega(d, d^U)$).

Towards this end, we draw inspiration from domain adaptation theory \citep{wu2019domain} and propose to use the following relaxed $f$-divergence to realize $\Omega(d, d^U)$:
\begin{definition}[Asymmetrically-relaxed $f$-divergence]\label{def:relaxed-f}
Given a constant $\beta > 1$ and a strictly convex and continuous function $f: \sR_+ \to \sR$ satisfying $f(1) = 0$, the asymmetrically-relaxed $f$-divergence between two distributions $p$ and $q$ (defined over domain $\gX$) is defined as:
\begin{align}
    D_{\Fbeta} (p \| q) = \int_\gX q(x) \Fbeta \left(\frac{p(x)}{q(x)}\right) \mathrm{d}x,
\end{align}
where $\widetilde{f}_\beta$ is a partially linearized function of $f$ defined as:
\begin{align}
    \Fbeta(u) = \begin{dcases}
        f(u) + C_{f, \beta} & \text{if}~~ u \geq \beta \\
        f'(\beta) u - f'(\beta) & \text{if}~~ u < \beta
    \end{dcases}
\end{align}
where the constant $C_{f, \beta} \defeq -f(\beta) + f'(\beta) (\beta - 1)$.
\end{definition}
It is worth noting that $\Fbeta$ is also continuous, convex (but not strictly convex) and satisfies $\Fbeta(1) = 0$. More importantly, $\DFbeta(p \| q) = 0$ if and only if $p(x)/q(x) \leq \beta, \forall x \in \gX$ (proof can be found in the appendix). This property is valuable for IL with suboptimal demonstrations:

\begin{proposition}\label{prop:preserve-optimal-policy}
Under Assumption~\ref{assumption:concentrability}, for any strictly convex function $f$, let $\Omega_1(d, d^U) = \DF(d\|d^U)$ and $\Omega_2(d, d^U) = \DFbeta(d\|d^U)$ with $\beta = B$. When the behavior policy is not optimal ($d^U \neq d^*$), then $\Omega_1$ is biased while $\Omega_2$ preserves the optimal policy (i.e. $d^* \notin \argmin_d \Omega_1(d, d^U)$ and $d^* \in \argmin_d \Omega_2(d, d^U)$)
\end{proposition}

Thus we propose to use the relaxed $f$-divergence to realize the regularization. Let $\Omega(d, d^U) = \DFbeta(d \| d^U)$ and we aim to solve the constrained optimization problem in Eq.~(\ref{eq:objective-constrained-optimization})-(\ref{eq:bellman-flow-constrained-optimization}) in an offline fashion. Apply a change of variable $\omega(s,a) = \frac{d(s,a)}{d^U(s,a)}$ to the Lagrangian of above constrained optimization, we can get the following optimization problem over $\omega$ and $v$ (with $v(s)$ being the Lagrange multipliers) (derivations can be found in the appendix):
{\small
\begin{align}
    & \max_{\omega \geq 0} \min_{v}~  L_{\alpha,\beta}(\omega, v) \defeq~ 
    (1 - \gamma) \E_{p_0(s)} [v(s)] + \label{eq:L-omega-v-minimax}\\
    & \E_{d^U(s,a)} \left[\omega(s,a) e_v(s,a) - \omega(s,a) \log(\omega(s,a)) - \alpha \Fbeta(\omega(s,a))\right] \nonumber
\end{align}}
Here, $e_v(s,a) \defeq \log \frac{d^E(s,a)}{d^U(s,a)} + \gamma (\gT v)(s,a) - v(s)$, where the density ratio $d^E / d^U$ can be estimated via Eq.~(\ref{eq:binary-classification})-(\ref{eq:link-function}) and $(\gT v) (s,a) \defeq \sum_{s'} T(s'|s,a) v(s')$. Note that \eqref{eq:L-omega-v-minimax} can be estimated only using offline datasets $\gD^E$ and $\gD^U$ (assuming $\gD^U$ contains a set of initial states sampled from $p_0$).

However, the nested min-max optimization in \eqref{eq:L-omega-v-minimax} usually results in unstable training in practice. To avoid this issue, we follow \citep{lee2021optidice,kim2021demodice} to assume that every state $s \in \gS$ is reachable for the given MDP $\gM$ and thus there exists a strictly feasible $\omega$ such that $\omega(s,a) = d(s,a)/d^U(s,a) > 0,~\forall s, a$. Since \eqref{eq:L-omega-v-minimax} is a convex optimization problem with strict feasibility, due to strong duality via Slater's condition \citep{boyd2004convex}, we know that:
\begin{align}
    \max_{\omega \geq 0} \min_{v}~ L_{\alpha,\beta}(\omega, v) = \min_{v} \max_{\omega \geq 0}~ L_{\alpha,\beta}(\omega, v)
\end{align}
By changing the max-min problem to min-max problem and using a particular convex function to instantiate the relaxed $f$-divergence, we can obtain the following closed-form solution for the inner maximization problem:

\begin{restatable}{theorem}{closedforminner}
\label{the:closed-form-1}
Let $\DFbeta$ be the relaxed $f$-divergence in Definition~\ref{def:relaxed-f}, with the associated convex function defined as $f(u) = u \log u$. Then the closed-form solution $\omega^*_{v}(s,a) = \argmax_{\omega \geq 0} L_{\alpha, \beta}(\omega, v)$ is:
\begin{align}
    & \omega^*_{v}(s,a) 
    = \label{eq:optimal-omega} \\
    & \begin{dcases}
        \exp\left(\frac{e_v(s,a)}{1 + \alpha} - 1\right) & \text{if}~~ \mathfrak{A}(s,a) \\
        \exp\left(e_v(s,a) - 1 - \alpha(\log \beta + 1)\right) & \text{otherwise}
    \end{dcases}
    \nonumber
\end{align}
where $\mathfrak{A}(s,a)$ denotes the event: $\frac{e_v(s,a)}{\alpha + 1} > \log \beta + 1$.
Define $h(\omega(s,a)) \defeq \omega(s,a) e_v(s,a) - \omega(s,a) \log (\omega(s,a)) - \alpha \Fbeta(\omega(s,a))$
such that $L_{\alpha, \beta}(\omega, v) = \E_{d^U(s,a)}[h(\omega(s,a))] + (1 - \gamma) \E_{p_0(s)}[v(s)]$. 
Then we have:
\begin{align*}
    & h(\omega^*_v(s,a)) = \\
    & \begin{dcases}
        (1 + \alpha) \exp\left(\frac{e_v(s,a)}{1 + \alpha} - 1\right) + C_1 & \text{if}~~ \mathfrak{A}(s,a) \\
        \exp\left(e_v(s,a) - 1 - \alpha(\log \beta + 1)\right) + C_2 & \text{otherwise}
    \end{dcases}
\end{align*}
where $C_1 = - \alpha C_{f, \beta}$ and $C_2= \alpha (\log \beta + 1)$ are constants w.r.t. $\omega$ and $v$.
\end{restatable}
Based on Theorem~\ref{the:closed-form-1}, \methodshort{} solves
$\hat{v}^* = \argmin_v L_{\alpha, \beta}(v) = L_{\alpha, \beta}(\omega^*_v, v)$, which provides us a tractable way to leverage a less conservative support regularization to effectively learn from potentially highly suboptimal offline data.

\subsection{\methodshort{} with Density Ratio Correction}\label{sec:relaxdice-drc}
As discussed before, given datasets $\gD^E$ and $\gD^U$, we can obtain an approximate density ratio $\hat{r}(s,a) \approx \frac{d^E(s,a)}{d^U(s,a)}$. Although we should not expect such an approximate density ratio to be accurate given limited samples, it is likely that the density-ratio-corrected occupancy measure $\hat{r} \cdot d^U$ is closer to 
the expert occupancy measure 
$d^E$ than $d^U$. Thus another rational choice for realizing the regularization $\Omega(d, d^U)$ in \eqref{eq:objective-constrained-optimization} is the relaxed $f$-divergence between $d$ and $\hat{r} \cdot d^U$.
With this goal, we derive an extension of our method, \methodshort{} with Density Ratio Correction (\methodshortdrc{}).

Let $\Omega(d, d^U) = \DFbeta(d \| \hat{r} \cdot d^U)$. Similar to the derivation of \methodshort{}, we apply a change of variable $\omega(s,a) = \frac{d(s,a)}{d^U(s,a)}$ to the Lagrangian of the constrained optimization problem in Eq.(\ref{eq:objective-constrained-optimization})-(\ref{eq:bellman-flow-constrained-optimization}) and with strong duality, we can obtain the following min-max optimization problem (derivations in the appendix):
{\small
\begin{align}
    & \min_{v} \max_{\omega \geq 0}~ L^\dagger_{\alpha,\beta}(\omega, v) \defeq~ (1 - \gamma) \E_{p_0(s)} [v(s)] + \label{eq:L-omega-v-minimax-drc} \\
    & \E_{d^U(s,a)}
    \left[\omega(s,a) \left( e_v(s,a) - \log(\omega(s,a)) \right) - \alpha \hat{r}(s,a) \Fbeta\left(\frac{\omega(s,a)}{\hat{r}(s,a)}\right)\right] \nonumber
\end{align}}
where $v(s)$ is the Lagrange multiplier and $e_v(s,a)$ is defined same as before.

Similar to \methodshort{}, we then introduce the following theorem to characterize the closed-form solution of the inner maximization problem in \eqref{eq:L-omega-v-minimax-drc} to avoid nested min-max optimization:
\begin{restatable}{theorem}{closedformdrc}
\label{the:closed-form-2}
Let $\DFbeta$ be the relaxed $f$-divergence in Definition~\ref{def:relaxed-f}, with $f(u) = u \log u$. Then the closed-form solution $\omega^*_{v}(s,a) = \argmax_{\omega \geq 0} L^\dagger_{\alpha, \beta}(\omega, v)$ is:

\begin{align}
    & \omega^*_{v}(s,a) 
    = \label{eq:optimal-omega-drc}\\
    & \begin{dcases}
        \exp\left(\frac{e_v(s,a) + \alpha \log \hat{r}(s,a)}{1 + \alpha} - 1\right)  & \text{if}~~ 
        \mathfrak{B}(s,a) \\
        \exp\left(e_v(s,a) - 1 - \alpha(\log \beta + 1)\right) & \text{otherwise}
    \end{dcases}
    \nonumber
\end{align}
where $\mathfrak{B}(s,a)$ denotes the event: $\frac{e_v(s,a) - \log \hat{r}(s,a)}{\alpha + 1} > \log \beta + 1$.
Define $h^\dagger(\omega(s,a)) \defeq \omega(s,a) e_v(s,a) - \omega(s,a) \log (\omega(s,a)) - \alpha \hat{r}(s,a) \Fbeta\left(\frac{\omega(s,a)}{\hat{r}(s,a)}\right)$ such that $L^\dagger_{\alpha, \beta}(\omega, v) = \E_{d^U(s,a)}[h^\dagger(\omega(s,a))] + (1 - \gamma) \E_{p_0(s)}[v(s)]$. Then we have:
{\small
\begin{align*}
    & h^\dagger(\omega^*_v(s,a)) = \\
    & \begin{dcases}
        (1 + \alpha) \exp\left(\frac{e_v(s,a) + \alpha \log \hat{r}(s,a)}{1 + \alpha} - 1\right) + C_3 
        & \text{if}~~
        \mathfrak{B}(s,a) \\
        \exp\left(e_v(s,a) - 1 - \alpha(\log \beta + 1)\right) + C_4 & \text{otherwise}
    \end{dcases}
\end{align*}
}
where $C_3 = - \alpha C_{f, \beta} \rhat(s,a)$ and $C_4= \alpha (\log \beta + 1) \rhat(s,a)$ are constants w.r.t. $\omega$ and $v$.
\end{restatable}
Based on Theorem~\ref{the:closed-form-2}, \methodshortdrc{} solves 
$\hat{v}^* = \argmin_v L^\dagger_{\alpha, \beta}(v) = L^\dagger_{\alpha, \beta}(\omega^*_v, v)$, which has the potential for better policy learning because of a more well-behaved regularization.

\subsection{Policy Extraction}
Given $v^*$, the corresponding density ratio $\hat{\omega}^*$ can be recovered according to \eqref{eq:optimal-omega} (for \methodshort{}) and \eqref{eq:optimal-omega-drc} (for \methodshortdrc{}) respectively. 
We can then use the following weighted BC objective (importance sampling or self-normalized importance sampling) for policy extraction:
\begin{align}
    & \max_\pi~ \E_{d^U(s,a)}[\hat{\omega}^*(s,a) \log \pi(a|s)] ~~\text{or} \label{eq:policy-extraction}\\
    & \max_\pi~ \frac{\E_{d^U(s,a)}[\hat{\omega}^*(s,a) \log \pi(a|s)]}{\E_{d^U(s,a)} [\hat{\omega}^*(s,a)]}
    \nonumber
\end{align}
In practice, we use samples from $\gD^U$ to estimate the expectations and we find that the latter one (using the self-normalized weight) tends to perform better, which we employ in our experiments.

\subsection{Practical Considerations}
Since we only have samples $(s,a,s')$ from the dataset $\gD^U$, similar to \citep{ValueDICE2019Kostrikov,lee2021optidice,kim2021demodice}, we have to use a single-point estimation $\hat{e}_v(s,a,s') = \log \frac{d^E(s,a)}{d^U(s,a)} + \gamma v(s') - v(s)$ for $e_v(s,a)$. This estimation is generally biased (when the MDP is stochastic) due to the non-linear exponential function outside of $e_v$. However, similar to the observation in \citep{ValueDICE2019Kostrikov}, we found this simple approach was enough to achieve good empirical performance on the standard benchmark domains we considered, thus we do not further use the Fenchel conjugate to remove the bias \citep{nachum2019dualdice}.

We use multilayer perceptron (MLP) networks to parametrize the classifier $c_\theta$ in \eqref{eq:binary-classification}, the Lagrange multiplier $v_\phi$, and the policy $\pi_\psi$ in \eqref{eq:policy-extraction}. 
Since the objectives $L_{\alpha, \beta}(v)$ and $L^\dagger_{\alpha, \beta}(v)$ contain exponential terms, we use gradient penalty \citep{gulrajani2017improved} to enforce Lipschitz constraints on networks $c_\theta$ and $v_\phi$, which can effectively stabilize the training. The required density ratio $d^E/d^U$ in $\hat{e}_v$ will be estimated via \eqref{eq:link-function} as $\rhat_\theta = \frac{c_\theta}{1 - c_\theta}$.

Regarding the hyperparameter $\beta$, in \methodshort{}, since ideally $\beta$ should be around the upper bound of $d^E/d^U$, we can automatically set it using the approximate density ratio $\rhat_\theta$ (e.g., by setting $\beta$ to be the running average of the maximum estimated density ratio of each minibatch); while in \methodshortdrc{}, $\beta$ should characterize the upper bound of the density ratio $d^E/ (\rhat_\theta \cdot d^U)$, which we expect to be small (e.g.$1.5$ or $2$) as $\rhat_\theta \cdot d^U$ is a density-ratio-corrected occupancy measure. In summary, \methodshort{} does not introduce new hyperparameter that requires tuning by automatically setting $\beta$ according to the data, while \methodshortdrc{} has the potential for better policy learning with the requirement of manually specifying $\beta$. More details of the practical implementations can be found in the appendix.
\section{Related Work}
\label{sec:rw}

\paragraph{Learning from imperfect demonstrations.} Imitation learning~\citep{pomerleau1988alvinn,Dagger2011Ross,GAIL2016Ho,Feedback2021Spencer} typically requires many optimal demonstrations, which could be expensive and time-consuming to collect. To address this limitation, imitation learning from imperfect demonstrations~\citep{wu2019imitation,brown2019extrapolating,brown2020better,brantley2019disagreement,tangkaratt2020variational,Wang2018SupervisedRL} arises as a promising alternative. 

To do so, prior works assume that the imperfect demonstrations consist of a mixture of expert data and suboptimal data and have considered two-step importance weighted IL~\citep{wu2019imitation}, learning imperfect demonstrations with adversarial training~\citep{wu2019imitation,wang2021learning} and training an ensemble of policies with weighted BC objectives~\citep{sasaki2020behavioral}. Note that \cite{wu2019imitation} assumes the access to the optimality labels in the imperfect demonstrations whereas \cite{wang2021learning} and \cite{sasaki2020behavioral} remove this strong assumption, which is followed in our work. Moreover, \cite{wu2019imitation} and \cite{wang2021learning} require online data collection for policy improvement while our work focuses on learning from offline data. 

The closest work to our setting is DemoDICE~\citep{kim2021demodice}, which performs offline imitation learning with a KL constraint to regularize the learned policy to stay close to the behavior policy. Such constraint can mitigate the distribution shift issue when learning from offline data~\citep{levine2020offline,kumar2019stabilizing,fujimoto2018off}, but could be overly conservative due to the exact distribution matching regularization especially when the imperfect data is highly suboptimal (see Proposition~\ref{prop:preserve-optimal-policy}). Our method mitigates this issue by instead using a support regularization. Although \citep{wu2019behavior} discussed a brief empirical exploration that using support regularization over policies offers little benefits, we instead formulate it as a constrained optimization over occupancy measures to take into consideration the diversity in both states and actions and observed clear practical benefits. Moreover, we surprisingly found that the increased complexity in minimax optimization can be resolved by the convenient closed-form solutions of the inner maximization problems.

\paragraph{Offline learning with stationary distribution correction.} Prior works in offline RL / IL have used distribution correction to mitigate distribtuion shift. AlgaeDICE \citep{nachum2019algaedice} leverages a dual formulation of $f$-divergence \citep{nachum2019dualdice} to regularize the stationary distribution besides the policy improvement objective in offline RL. ValueDICE~\citep{ValueDICE2019Kostrikov} uses a similar formulation to AlgaeDICE for off-policy distribution matching with expert demonstrations. However, both ValueDICE and AlgaeDICE need to solve the nested min-max optimization problem, which is usually unstable in practice. OptiDICE~\citep{lee2021optidice} and DemoDICE~\citep{kim2021demodice} resolve this issue via deriving a closed-form solution of their inner optimization problem. Our method also enjoys the same desired property while using an asymmetrically-relaxed $f$-divergence~\citep{wu2019domain} as a more appropriate regularization in face of highly suboptimal offline data.

\section{Experiments}\label{sec:experiments}

In our experiments, we aim to answer the following three questions: (1) how do \methodshort{} and \methodshortdrc{} compare to prior works on standard continuous-control tasks using limited expert data and suboptimal offline data? (2) can \methodshort{} remain superior performance compared to prior methods as the quality of the suboptimal offline dataset decreases? (3) can \methodshort{} behave more robustly with respect to different hyperparameter choices compared to prior methods?

\begin{table*}[ht]
\begin{center}
\resizebox{\textwidth}{!}{\begin{tabular}{l|l|l|l|l|l|ll|lll}
\toprule
& & & & & & \multicolumn{2}{c}{\textbf{BC-DRC}} & \multicolumn{3}{c}{\textbf{BC}} \\
\textbf{Envs} & \textbf{Tasks} & \textbf{RelaxDICE} & \textbf{\methodshortdrc{}} & \textbf{DemoDICE} & \textbf{BCND} & $\eta = 0.0$ & $\eta = 0.5$ & $\eta = 0.0$ & $\eta = 0.5$ & $\eta = 1.0$\\ \midrule
& \texttt{L1} & \textbf{74.6}$\pm$9.1 & \textbf{73.6}$\pm$6.3 & 70.9$\pm$9.0 & 6.6$\pm$2.9 & 1.4$\pm$1.1 & 2.9$\pm$3.9 & 1.8$\pm$1.2 & 7.6$\pm$8.0 & 17.8$\pm$11.7\\
& \texttt{L2} & \textbf{64.2}$\pm$8.7 & \textbf{70.0}$\pm$13.7 & 54.4$\pm$6.4 & 4.8$\pm$4.2 & 2.4$\pm$1.2 & 4.9 $\pm$ 2.7 & 2.9$\pm$2.1 & 1.6$\pm$1.5 & 17.8$\pm$11.7\\
\texttt{hopper}& \texttt{L3} & \textbf{36.2}$\pm$5.6 & \textbf{41.5}$\pm$4.3 & 31.4$\pm$9.7 & 2.3$\pm$2.3 & 1.8$\pm$0.5 &1.5 $\pm$ 0.6 & 5.0$\pm$3.3 & 1.4$\pm$0.9 & 17.8$\pm$11.7\\
& \texttt{L4} & \textbf{38.7}$\pm$8.2 & \textbf{40.2}$\pm$6.9 & 34.9$\pm$5.6 & 0.9$\pm$0.3 & 0.7$\pm$0.2 & 1.5$\pm$0.7 & 0.8$\pm$0.4 & 0.8$\pm$0.3 & 17.8$\pm$11.7\\
\midrule
& \texttt{L1} & \textbf{59.1}$\pm$8.6 & \textbf{66.7}$\pm$5.1 & 58.6$\pm$8.0 & 2.5$\pm$0.1 & 2.6 $\pm$0.0 & 2.6$\pm$0.0 & 2.6$\pm$0.0 & 2.6$\pm$0.0 & 0.9 $\pm$ 1.1\\
& \texttt{L2} & \textbf{49.3}$\pm$4.7 & \textbf{52.1}$\pm$2.0 & 48.3$\pm$3.9 & 2.5$\pm$0.1 & 2.6 $\pm$0.0 & 2.6$\pm$0.0 & 2.6$\pm$0.0 & 2.6$\pm$0.0 & 0.9 $\pm$ 1.1\\
\texttt{halfcheetah}& \texttt{L3} & \textbf{35.0}$\pm$6.6 & \textbf{37.9}$\pm$4.0 & 32.9$\pm$2.1 & 2.5$\pm$0.0 & 2.6 $\pm$0.0 & 2.6$\pm$0.0 & 2.6$\pm$0.0 & 2.6$\pm$0.0 & 0.9 $\pm$ 1.1\\
& \texttt{L4} & \textbf{13.3}$\pm$2.1 & \textbf{16.1}$\pm$4.1 & 10.5$\pm$0.5 & 2.6$\pm$0.0 & 2.6 $\pm$0.0 & 2.6$\pm$0.0 & 2.6$\pm$0.0 & 2.6$\pm$0.0 & 0.9 $\pm$ 1.1\\
\midrule
& \texttt{L1} & 92.6$\pm$7.6 & \textbf{99.4}$\pm$1.9 & \textbf{98.8}$\pm$1.6 & 3.0$\pm$3.9 & 0.6$\pm$0.6 & 0.3$\pm$0.2 & 2.2$\pm$1.1 & 0.2$\pm$0.0 & 8.5$\pm$3.6\\
& \texttt{L2} & \textbf{69.7}$\pm$20.3 & \textbf{57.7}$\pm$12.5 & 42.1$\pm$23.9 & 0.1$\pm$0.2 & 0.2$\pm$0.0 & 0.5$\pm$0.3 & 0.2$\pm$0.0 & 0.3$\pm$0.1 & 8.5$\pm$3.6\\
\texttt{walker2d}& \texttt{L3} & \textbf{41.9}$\pm$23.1 & \textbf{56.7}$\pm$26.2 & 23.4$\pm$20.6 & 0.7$\pm$0.6 & 0.8$\pm$1.1 & 0.4$\pm$0.2 & 0.2$\pm$0.0 & 0.3$\pm$0.1 & 8.5$\pm$3.6\\
& \texttt{L4} & 26.3$\pm$17.2 & \textbf{49.5}$\pm$17.4 & \textbf{39.8}$\pm$22.4 & 0.2$\pm$0.2 & 0.2$\pm$0.1 & 0.5$\pm$0.4 & 0.2$\pm$0.1 & 0.2$\pm$0.1 & 8.5$\pm$3.6\\
\midrule
& \texttt{L1} & \textbf{91.9}$\pm$3.6 & \textbf{89.0}$\pm$5.3 & 77.9$\pm$8.7 & 12.2$\pm$2.4 & 61.7$\pm$4.9 & 21.3$\pm$1.2 & 66.2$\pm$11.2 & 21.3$\pm$1.1 & -8.3$\pm$4.3\\
& \texttt{L2} & \textbf{75.2}$\pm$5.8 & \textbf{82.1}$\pm$7.1 & 70.5$\pm$4.2 & 15.6$\pm$2.3 & 50.8$\pm$6.1 & 20.6$\pm$1.8 & 54.4$\pm$4.9 & 19.9$\pm$1.6 & -8.3$\pm$4.3\\
\texttt{ant}& \texttt{L3} & \textbf{58.7}$\pm$7.1 & \textbf{59.6}$\pm$9.0 & 49.9$\pm$2.9 & 17.0$\pm$0.9 & 38.1$\pm$5.2 & 18.8$\pm$4.7 & 37.6$\pm$3.0 & 22.6$\pm$0.1 & -8.3$\pm$4.3\\
& \texttt{L4} & \textbf{43.2}$\pm$7.2 & \textbf{41.3}$\pm$4.0 & -5.3$\pm$41.4 & 13.6$\pm$1.7 & 29.0$\pm$3.9 & 22.4$\pm$0.2 & 28.0$\pm$3.1 & 22.3$\pm$0.3 & -8.3$\pm$4.3\\
\midrule
\midrule
& \texttt{L1} & \textbf{24.2}$\pm$17.6 & \textbf{27.3}$\pm$13.9 & 4.1$\pm$3.6 & 0.2$\pm$0.0 & 0.4$\pm$0.1 & 0.3$\pm$0.0 & 0.4$\pm$0.0 & 0.3$\pm$0.0 & 4.8$\pm$2.9\\
\texttt{hammer}& \texttt{L2} & \textbf{18.3}$\pm$12.0 & \textbf{18.4}$\pm$13.6 & 17.3$\pm$8.9 & 0.2$\pm$0.0 & 0.3$\pm$0.0 & 0.5$\pm$0.4 & 0.3$\pm$0.1 & 0.3$\pm$0.0 & 4.8$\pm$2.9\\
& \texttt{L3} & \textbf{19.5}$\pm$15.5 & \textbf{20.6}$\pm$12.8 & 14.1$\pm$10.2 & 0.2$\pm$0.0 & 0.4$\pm$0.0 & 0.3$\pm$0.0 & 0.4$\pm$0.0 & 0.3$\pm$0.0 & 4.8$\pm$2.9\\
\midrule
& \texttt{L1} & \textbf{48.0}$\pm$2.3 & \textbf{50.9}$\pm$5.3 & 40.9$\pm$13.7 & -0.1$\pm$0.0 & -0.2$\pm$0.0 & 6.4$\pm$2.0 & -0.2$\pm$0.0 & 7.3$\pm$2.2 & 0.2$\pm$0.3\\
\texttt{relocate}& \texttt{L2} & \textbf{45.1}$\pm$5.7 & \textbf{52.4}$\pm$7.7 & 43.0$\pm$8.2 & -0.2$\pm$0.0 & 5.0$\pm$2.4 & -0.2$\pm$0.0 & -0.2$\pm$0.0 & 4.5$\pm$2.1 & 0.2$\pm$0.3\\
& \texttt{L3} & \textbf{39.0}$\pm$4.5 & \textbf{43.3}$\pm$7.8 & 27.1$\pm$6.9 & -0.2$\pm$0.0 & -0.2$\pm$0.0 & 2.9$\pm$1.9 & -0.2$\pm$0.0 & 3.3$\pm$3.0 & 0.2$\pm$0.3\\
\bottomrule
\end{tabular}}
\end{center}
\caption{\footnotesize Results for four MuJoCo environments \texttt{halfcheetah, hopper, walker2d} and \texttt{ant} and two Adroit environments \texttt{hammer} and \texttt{relocate} from D4RL~\citep{fu2020d4rl}. Numbers are averaged across 5 seeds, $\pm$ the 95\%-confidence interval. We bold the top 2 highest performances. Either \methodshort{} or \methodshortdrc{} achieves the best performance in each of 22 settings and outperforms the strongest baseline DemoDICE by a large margin in \texttt{L3} and \texttt{L4} settings where the offline data is highly suboptimal, suggesting the importance of using a relaxed distribution matching regularization.}
\normalsize
\label{tbl:mujoco_results}
\end{table*}

\paragraph{Environments, Datasets and Task Construction.}
In order to answer these questions, we consider offline datasets of four MuJoCo~\citep{todorov2012mujoco} locomotion environments (\texttt{hopper}, \texttt{halfcheetah}, \texttt{walker2d} and \texttt{ant}) and two Adroit robotic manipulation environments (\texttt{hammer} and \texttt{relocate}) from the standard offline RL benchmark D4RL~\citep{fu2020d4rl}. To construct settings where we have varying data quality of the suboptimal offline dataset, for MuJoCo tasks, we use $1$ trajectory from the \texttt{expert-v2} datasets as $\gD^E$ for each environment and create the suboptimal offline data $\gD^U$ by mixing $N^E$ transitions from \texttt{expert-v2} datasets and $N^R$ transitions from \texttt{random-v2} datasets with $4$ different ratios. We denote these settings as \texttt{L1} (Level 1), \texttt{L2} (Level 2), \texttt{L3} (Level 3) and \texttt{L4} (Level 4), which correspond to $\frac{N^E}{N^R} \approx 0.2, 0.15, 0.1, 0.05$ respectively. The higher the level, the more 
challenging the setting is. Note that all of the four settings are of much more suboptimal data composition compared to the data configuration used for $\gD^U$ adopted in \cite{kim2021demodice}, where $\gD^U$ in the most imperfect setting 
can have $\frac{N^E}{N^R} > 10.0$, e.g. on \texttt{walker2d}. 
The rationale of constructing such challenging datasets is that in practice, it is much cheaper to generate suboptimal and even random data and therefore a successful offline IL method should be equipped with the capacity of tackling these suboptimal offline datasets. In order to excel at \texttt{L1}, \texttt{L2}, \texttt{L3} and \texttt{L4}, a successful algorithm must effectively leverage $\gD^U$ to provide proper regularization for policy optimization. For Adroit tasks, following similar design choice, we construct three levels of data compositions, i.e. \texttt{L1}, \texttt{L2} and \texttt{L3}. Please see the appendix for details.

\paragraph{Comparisons.} To answer these questions, we first consider the following prior approaches. We compare RelaxDICE to \textbf{DemoDICE}~\citep{kim2021demodice}, which performs offline imitation learning with supplementary imperfect demonstrations via applying a KL constraint between the occupancy measure of the learned policy and that of the behavior policy. We also consider \textbf{BCND}~\citep{sasaki2020behavioral} as a baseline, which learns an ensemble of policies via a weighted BC objective on noisy demonstrations. Moreover, we compare to \textbf{BC}$(\eta)$~\citep{kim2021demodice}, where $\eta \in \{0, 0.5, 1.0\}$ corresponds to a weight factor that balances between minimizing the negative log-likelihood on expert data $\gD^E$ and minimizing the negative log-likelihood on
suboptimal offline data $\gD^U$: $\min_{\pi} L_{\text{BC}(\eta)}(\pi) \defeq - \eta\frac{1}{|\gD^E|}\sum_{(s, a)\in \gD^E} \log \pi(a|s) - (1-\eta)\frac{1}{|\gD^U|}\sum_{(s, a) \in \gD^U} \log \pi(a|s)$.

Finally, we consider the importance-weighted BC$(\eta)$ denoted as \textbf{BC-DRC}$(\eta)$, i.e. BC with density ratio correction, where we train a classifier to approximate the density ratio $d^E/d^U$ as $\rhat$ via Eq.~(\ref{eq:binary-classification})-(\ref{eq:link-function}), and perform weighted BC$(\eta)$ using $\rhat$ as the importance weights: $\min_{\pi} L_{\text{BC-DRC}(\eta)}(\pi) \defeq -\eta\frac{1}{|\gD^E|}\sum_{(s, a) \in \gD^E} \log \pi(a|s) - (1-\eta)\frac{1}{|\gD^U|}\sum_{(s, a)\in \gD^U} \rhat(s,a) \log \pi(a|s)$.

For all the tasks, we use $\alpha=0.2$ for \methodshort{} and use $\alpha=0.05$ for DemoDICE as suggested in \cite{kim2021demodice}, which is also verified in our experiments. We pick $\alpha$ and $\beta$ for \methodshortdrc{} via grid search, which we will discuss in the appendix. For more details of the experiment set-ups, evaluation protocols, hyperparameters and practical implementations, please see the appendix.

\subsection{Results of Empirical Evaluations}
\label{sec:empirical_results}

To answer question (1) and (2), we evaluate \methodshort{}, \methodshortdrc{} and other approaches discussed above on 6 D4RL environments (4 MuJoCo locomotion tasks and 2 Adroit robotic manipulation tasks) with 22 different settings in total.
We present the full results in Table~\ref{tbl:mujoco_results}.

As shown in Table~\ref{tbl:mujoco_results}, 
\methodshortdrc{} achieves the best performance in 18 out of 22 tasks whereas \methodshort{} excels in the remaining 4 settings. It is also worth noting that \methodshort{} outperforms the strongest baseline DemoDICE in 20 out of 22 settings, without requiring tuning two hyperparameters as in \methodshortdrc{}. 
Overall, we observe that the best performing method (either \methodshort{} or \methodshortdrc{}) achieves over 30\% performance improvement on average over DemoDICE.
Moreover, in settings where the offline data is highly suboptimal, e.g. \texttt{L3} and \texttt{L4}, both \methodshort{} and \methodshortdrc{} can significantly outperform DemoDICE
except on \texttt{walker2d-L4} where \methodshort{} is a bit worse than DemoDICE but \methodshortdrc{} prevails. In particular, on high-dimensional locomotion tasks such as \texttt{ant} and complex manipulation tasks such as \texttt{hammer}, \methodshort{} and \methodshortdrc{} outperform DemoDICE by a significant margin on hard datasets such as \texttt{L3} and \texttt{L4}. These suggest that using a less conservative support regularization can be crucial in cases with extremely low-quality offline data, supporting our theoretical analysis.

\subsection{Sensitivity of Hyperparameters
}
\label{sec:ablations}
To answer question (3), we perform an ablation study on the sensitivity of the hyperparameter $\alpha$ in \methodshort{} and DemoDICE~\citep{kim2021demodice}, which controls the strength of the regularization between the learned policy and the behavior policy. We pick two continuous-control tasks \texttt{halfcheetah} and \texttt{walker2d}
and evaluate the performance of \methodshort{} and DemoDICE using $\alpha \in \{0.05, 0.1, 0.2, 0.3, 0.4, 0.5\}$ on all four settings in each of the two tasks. As shown in Figure~\ref{fig:ablation} 
in the appendix, \methodshort{} is much more robust w.r.t. $\alpha$ compared to DemoDICE in all of the 8 scenarios as \methodshort{} remains roughly a flat line in all eight plots and the performance of DemoDICE drops significantly as $\alpha$ increases. We think the reason is that DemoDICE employs a conservative exact distribution matching constraint and therefore requires different values of $\alpha$ on datasets with different data quality to find the delicate balance between policy optimization based on limited $\gD^E$ and regularization from suboptimal $\gD^U$, e.g. higher $\alpha$ when the data quality is high and lower $\alpha$ when the data is highly suboptimal. In contrast, \methodshort{} imposes a relaxed support regularization, which is less conservative and therefore less sensitive w.r.t. data quality. Since tuning hyperparameters for offline IL / RL in a fully offline manner remains an open problem and often requires expensive online samples~\citep{monier2020offline,kumar2021should,yu2021combo,kurenkov2021showing}, we believe \methodshort{}'s robustness w.r.t. the hyperparameters should significantly benefit practitioners.

\section{Conclusion and Discussion}\label{sec:discussion}
We present \methodshort{}, a novel offline imitation learning methods for learning policies from limited expert data and supplementary imperfect data. Different from prior works using regularizations originally designed for exact distribution matching, we employ an asymmetrically relaxed $f$-divergence as a more forgiving regularization that proves effective even for settings where the imperfect data is highly suboptimal. Both \methodshort{} and its extension \methodshortdrc{} can avoid unstable min-max optimization of the regularized stationary state-action distribution matching problem by supporting a closed-form solution of the inner maximization problem, and show superior performance to strong baselines in our extensive empirical study. 

\section{Acknowledgments}
This work was supported by NSF (\#1651565), AFOSR (FA95501910024), ARO (W911NF-21-1-0125), ONR, DOE, CZ Biohub, and Sloan Fellowship.

\bibliography{bib}

\newpage
\appendix
\onecolumn

\section{Derivation for \methodshort{}}\label{app:derivation-relaxdice}
We propose to use the relaxed $f$-divergence to realize the regularization $\Omega(d, d^U)$ and aim to solve the following constrained optimization problem \emph{in an offline fashion}:
\begin{align}
    & \max_{d \geq 0}~ - \KL(d \| d^E) - \alpha \DFbeta(d \| d^U)
    \label{eq-app:objective-constrained-optimization-relaxed}\\
    & \mathrm{s.t.}~~ \sum_a d(s,a) = (1 - \gamma) p_0(s) + \gamma \sum_{s', a'} T(s|s',a') d(s', a'), \forall s \in \gS. 
    \label{eq-app:bellman-flow-constrained-optimization-relaxed}
\end{align}
For notation simplicity, we define $(\gB d) (s) \defeq \sum_a d(s,a)$, $(\gT d) (s) \defeq \sum_{s', a'} T(s|s',a') d(s', a')$ and $(\gT v) (s,a) \defeq \sum_{s'} T(s'|s,a) v(s')$. 
First, we can obtain the following Lagrangian for above constrained optimization problem (with $v(s)$ being the Lagrange multipliers):
\begin{align}
    \max_{d \geq 0} \min_{v}~ L_{\alpha, \beta}(d, v) \defeq
    & - \KL(d \| d^E) - \alpha \DFbeta(d \| d^U) \nonumber
    \\
    & + \sum_{s} v(s)((1 - \gamma) p_0(s) + \gamma (\gT d) (s) - (\gB d) (s)) \label{eq-app:lagrangian}
\end{align}
Plugging in the definitions of KL divergence and relaxed $f$-divergence in Definition~\ref{def:relaxed-f}, $L_{\alpha, \beta}(d, v)$ in \eqref{eq-app:lagrangian} can be written as:
\begin{align}
    L_{\alpha, \beta}(d, v) =&~ \E_{d(s,a)} \left[\log \frac{d^E(s,a) \cdot d^U(s,a)}{d^U(s,a) \cdot d(s,a)}\right] - \alpha \E_{d^U(s,a)} \left[\Fbeta\left(\frac{d(s,a)}{d^U(s,a)}\right)\right] \nonumber\\
    & + (1 - \gamma) \E_{p_0(s)} [v(s)] + \E_{d(s,a)} [\gamma (\gT v)(s,a) - v(s)] 
    \label{eq-app:L-alpha-beta-1}
    \\
    =&~ \E_{d^U(s,a)} \left[ \frac{d(s,a)}{d^U(s,a)} \left(\log \frac{d^E(s,a)}{d^U(s,a)} + \gamma (\gT v)(s,a) - v(s) - \log \frac{d(s,a)}{d^U(s,a)}\right) \right] \nonumber\\
    & - \alpha \E_{d^U(s,a)} \left[\Fbeta\left(\frac{d(s,a)}{d^U(s,a)}\right)\right] + (1 - \gamma) \E_{p_0(s)} [v(s)]
    \label{eq-app:L-alpha-beta-2}
\end{align}
where \eqref{eq-app:L-alpha-beta-1} uses the fact that 
$\sum_s v(s) (\gT d)(s) = \sum_{s,a} d(s,a) (\gT v)(s,a)$,
and the density ratio $d^E / d^U$ can be estimated via Eq.~(\ref{eq:binary-classification})-(\ref{eq:link-function}); and \eqref{eq-app:L-alpha-beta-2} uses importance sampling to change the expectation w.r.t. $d$ to an expectation w.r.t. $d^U$ for offline learning. 

Define $e_v(s,a) \defeq \log \frac{d^E(s,a)}{d^U(s,a)} + \gamma (\gT v)(s,a) - v(s)$ and use a change of variable $\omega(s,a) = \frac{d(s,a)}{d^U(s,a)}$, we obtain the following optimization problem:
\begin{align}
    \max_{\omega \geq 0} \min_{v}~ L_{\alpha,\beta}(\omega, v) \defeq~ & \E_{d^U(s,a)} \left[\omega(s,a) e_v(s,a) - \omega(s,a) \log(\omega(s,a)) - \alpha \Fbeta(\omega(s,a))\right] \nonumber \\
    & + (1 - \gamma) \E_{p_0(s)} [v(s)] 
    \label{eq-app:L-omega-v-minimax}
\end{align}
which can be estimated only using offline datasets $\gD^E$ and $\gD^U$ (assuming $\gD^U$ additionally contains a set of initial states sampled from $p_0$).

\textbf{Remark.} We note that DemoDICE~\citep{kim2021demodice} is a special case of \methodshort{} when $\beta \to 0$, which can be verified according to Definition~\ref{def:relaxed-f} and Theorem~\ref{the:closed-form-1} (we will always be in the first condition of the piecewise function as $\log \beta + 1 \to - \infty$ when $\beta \to 0$).

\section{Derivation for \methodshortdrc{}}\label{app:derivation-relaxdice-drc}
As discussed before, another attractive choice for realizing the regularization is the relaxed $f$-divergence between $d$ and the density-ratio-corrected behavior occupancy measure $\rhat \cdot d^U$.

Let $\Omega(d, d^U) = \DFbeta(d \| \hat{r} \cdot d^U)$ and we aim to solve the following constrained optimization problem \emph{in an offline fashion}:
\begin{align}
    & \max_{d \geq 0}~ - \KL(d \| d^E) - \alpha \DFbeta(d \| \hat{r} \cdot d^U)
    \label{eq-app:objective-constrained-optimization-relaxed-drc}\\
    & \mathrm{s.t.}~~ \sum_a d(s,a) = (1 - \gamma) p_0(s) + \gamma \sum_{s', a'} T(s|s',a') d(s', a'), \forall s \in \gS. 
    \label{eq-app:bellman-flow-constrained-optimization-relaxed-drc}
\end{align}

Similar to the derivation of \methodshort{}, we can obtain the following Lagrangian for the constrained optimization problem in Eq.(\ref{eq-app:objective-constrained-optimization-relaxed-drc})-(\ref{eq-app:bellman-flow-constrained-optimization-relaxed-drc}) (with $v(s)$ being the Lagrange multipliers):
\begin{align}
    L^\dagger_{\alpha, \beta}(d, v)=&~ \E_{d^U(s,a)} \left[ \frac{d(s,a)}{d^U(s,a)} \left(e_v(s,a) - \log \frac{d(s,a)}{d^U(s,a)}\right) \right] \nonumber \\
    &- \alpha \E_{d^U(s,a)} \left[
    \hat{r}(s,a) \cdot
    \Fbeta\left(\frac{d(s,a)}{\hat{r}(s,a) \cdot d^U(s,a)}\right)\right] + (1 - \gamma) \E_{p_0(s)} [v(s)]
\end{align}

Similarly, we use a change of variable $\omega(s,a) = \frac{d(s,a)}{d^U(s,a)}$ and apply strong duality to obtain the following min-max problem over $\omega$:
\begin{align}
    \min_{v} \max_{\omega \geq 0}~ L^\dagger_{\alpha,\beta}(\omega, v) \defeq~& \E_{d^U(s,a)} \left[\omega(s,a) e_v(s,a) - \omega(s,a) \log(\omega(s,a)) - \alpha \hat{r}(s,a) \Fbeta\left(\frac{\omega(s,a)}{\hat{r}(s,a)}\right)\right] \nonumber \\
    & + (1 - \gamma) \E_{p_0(s)} [v(s)] 
    \label{eq-app:L-omega-v-minimax-drc}
\end{align}

\section{Proofs}\label{app:proofs}
\begin{lemma}
For distributions $p$ and $q$ defined on domain $\gX$, if $\frac{p(x)}{q(x)} < \beta,~\forall x \in \gX$, then the relaxed $f$-divergence $\DFbeta(p \| q) = 0$.
\end{lemma}
\begin{proof}
According to Definition~\ref{def:relaxed-f}, if $p/q < \beta$ everywhere, we have:
\begin{align*}
    \DFbeta(p, \| q) &= \int_{\gX} q(x) \Fbeta\left(\frac{p(x)}{q(x)}\right) \mathrm{d}x \\
    &= \int_{\gX} q(x) \left(f'(\beta) \frac{p(x)}{q(x)} - f'(\beta)\right) \mathrm{d}x \\
    &= f'(\beta) \int_{\gX} p(x) \mathrm{d}x - f'(\beta) \int_{\gX} q(x) \mathrm{d}x = 0
\end{align*}
\end{proof}

\closedforminner*
\begin{proof}
Since $\Fbeta$ is a continuous piecewise function, $h(\omega(s,a))$ is also a continuous piecewise function:
\begin{align*}
    h(\omega(s,a)) &= \begin{dcases}
        e_v(s,a) \omega(s,a) - \omega(s,a) \log(\omega(s,a)) - \alpha f(\omega(s,a)) - \alpha C_{f, \beta} & \text{if}~~ \omega(s,a) \geq \beta\\
        e_v(s,a) \omega(s,a) - \omega(s,a) \log(\omega(s,a)) - \alpha f'(\beta) \omega(s,a) + \alpha f'(\beta) & \text{if}~~ \beta \geq \omega(s,a) \geq 0
    \end{dcases}
\end{align*}

When $f(u) = u \log u$ and $f'(u) = \log u + 1$, the gradient of $h(\omega(s,a))$ is given by:
\begin{align*}
    h'(\omega(s,a)) &= \begin{dcases}
        e_v(s,a) - (\alpha + 1) (\log(\omega(s,a)) + 1) & \text{if}~~ \omega(s,a) \geq \beta\\
        e_v(s,a) - \log(\omega(s,a)) - 1 - \alpha (\log \beta + 1) & \text{if}~~ \beta \geq \omega(s,a) \geq 0
    \end{dcases}
\end{align*}

Define $\omega^*_{\leq \beta}(s,a)$ and $\omega^*_{\geq \beta}(s,a)$ as the local maximum for $[0, \beta]$ and $[\beta, +\infty)$ respectively:
\begin{align*}
    \omega^*_{\leq \beta}(s,a) \defeq \argmax_{\beta \geq \omega(s,a) \geq 0} h(\omega(s,a)) ~~\text{and}~~ \omega^*_{\geq \beta}(s,a) \defeq \argmax_{\omega(s,a) \geq \beta} h(\omega(s,a))
\end{align*}
and the overall maximum will be either $\omega^*_{\leq \beta}(s,a)$ or $\omega^*_{\geq \beta}(s,a)$ depending on whose function value is larger (the global maximum must be one of the local maximums).

In the following, we will use the fact that $h'$ is strictly decreasing due to the strict concavity of $h$ (an affine function plus a strictly concave function).

\textbf{(1) When $e_v(s,a) > (\alpha + 1) (\log \beta + 1)$ (or equivalently $h'(\beta) > 0$)}:

For $\beta \geq \omega(s,a) \geq 0$, we know that $h'(\omega(s,a)) \geq h'(\beta) = e_v(s,a) - (\alpha + 1)(\log \beta + 1) > 0$, so $\omega^*_{\leq \beta}(s,a) = \beta$.

For $\omega(s,a) \geq \beta$, since $h'$ is strictly decreasing and $h'(\beta) > 0$, we know that $\omega^*_{\geq \beta}(s,a)$ is attained at $h'(\omega(s,a)) = e_v(s,a) - (\alpha + 1) (\log(\omega(s,a)) + 1) = 0$. Thus $\omega^*_{\geq \beta}(s,a) = \exp\left(\frac{e_v(s,a)}{1 + \alpha} - 1\right)$. 

Moreover, because $h'(\omega(s,a)) > 0$ when $\beta \leq \omega(s,a) < \omega^*_{\geq \beta}(s,a)$, we know that $h(\omega^*_{\leq \beta}(s,a)) = h(\beta) < h(\omega^*_{\geq \beta}(s,a))$, and the overall maximum is $\omega^*(s,a) = \exp\left(\frac{e_v(s,a)}{1 + \alpha} - 1\right) > \beta$.

In this case, the maximum function value of $h$ is:
\begin{align*}
    h(\omega^*(s,a)) &= e_v(s,a) \omega^*(s,a) - \omega^*(s,a) \log(\omega^*(s,a)) - \alpha f(\omega^*(s,a)) - \alpha C_{f, \beta} \\
    &= \omega^*(s,a) \left(e_v(s,a) - (\alpha + 1) \log (\omega^*(s,a))\right) - \alpha C_{f, \beta} \\
    &= (\alpha + 1) \omega^*(s,a) - \alpha C_{f, \beta} \\
    &= (\alpha + 1) \exp\left(\frac{e_v(s,a)}{1 + \alpha} - 1\right) - \alpha C_{f, \beta}
\end{align*}

\textbf{(2) When $e_v(s,a) \leq (\alpha + 1) (\log \beta + 1)$ (or equivalently $h'(\beta) \leq 0$)} :

Since $\lim_{\omega(s,a) \to 0} h'(\omega(s,a)) = + \infty$ and $h'(\beta) \leq 0$, $\omega^*_{\leq \beta}(s,a)$ is attained at $h'(\omega(s,a)) = e_v(s,a) - \log(\omega(s,a)) - 1 - \alpha (\log \beta + 1) = 0$. Thus $\omega^*_{\leq \beta}(s,a) = \exp\left(e_v(s,a) - 1 - \alpha(\log \beta + 1)\right)$.

For $\omega(s,a) \geq \beta$, since $h'$ is strictly decreasing and $h'(\beta) \leq 0$, so $h'(\omega(s,a)) \leq 0$ and $\omega^*_{\geq \beta}(s,a) = \beta$.

Moreover, because $h(\omega^*_{\leq \beta}(s,a)) \geq h(\beta) = h(\omega^*_{\geq \beta}(s,a))$, the overall maximum is $\omega^*(s,a) = \exp\left(e_v(s,a) - 1 - \alpha(\log \beta + 1)\right) \leq \beta$.

In this case, the maximum function value of $h$ is:
\begin{align*}
    h(\omega^*(s,a)) &= e_v(s,a) \omega^*(s,a) - \omega^*(s,a) \log(\omega^*(s,a)) - \alpha f'(\beta) \omega^*(s,a) + \alpha f'(\beta) \\
    &= \omega^*(s,a) \left(e_v(s,a) - \log (\omega^*(s,a)) - \alpha (\log \beta + 1) \right) + \alpha (\log \beta + 1) \\
    &= \omega^*(s,a) + \alpha (\log \beta + 1) \\
    &= \exp\left(e_v(s,a) - 1 - \alpha(\log \beta + 1)\right) + \alpha (\log \beta + 1)
\end{align*}
\end{proof}

\closedformdrc*
\begin{proof}
Since $\Fbeta$ is a continuous differentiable piecewise function, $h^\dagger(\omega(s,a))$ is also a continuous differentiable piecewise function. When $f(u) = u \log u$ and $f'(u) = \log u + 1$, we have:
\begin{align*}
    h^\dagger(\omega(s,a)) &= \begin{dcases}
        e_v(s,a) \omega(s,a) - \omega(s,a) \log(\omega(s,a)) - \alpha \omega(s,a) \log \left( \frac{\omega(s,a)}{\hat{r}(s,a)} \right)  - \alpha C_{f, \beta} \hat{r}(s,a) & \text{if}~~ \frac{\omega(s,a)}{\hat{r}(s,a)} \geq \beta\\
        e_v(s,a) \omega(s,a) - \omega(s,a) \log(\omega(s,a)) - \alpha f'(\beta) \omega(s,a) + \alpha f'(\beta) \hat{r}(s,a) & \text{if}~~ \beta \geq \frac{\omega(s,a)}{\hat{r}(s,a)} \geq 0
    \end{dcases}
\end{align*}
The gradient of $h^\dagger(\omega(s,a))$ is given by:
\begin{align*}
    (h^\dagger)'(\omega(s,a)) &= \begin{dcases}
        e_v(s,a) - (\alpha + 1) (\log(\omega(s,a)) + 1) + \alpha \log \hat{r}(s,a) & \text{if}~~ \omega(s,a) \geq \beta \rhat(s,a)\\
        e_v(s,a) - \log(\omega(s,a)) - 1 - \alpha (\log \beta + 1) & \text{if}~~ \beta \rhat(s,a) \geq \omega(s,a) \geq 0
    \end{dcases}
\end{align*}

Define $\omega^*_{\leq \beta \hat{r}}(s,a)$ and $\omega^*_{\geq \beta \hat{r}}(s,a)$ as the local maximum for $[0, \beta \hat{r}(s,a)]$ and $[\beta \hat{r}(s,a), +\infty)$:
\begin{align*}
    \omega^*_{\leq \beta \rhat}(s,a) \defeq \argmax_{\beta \rhat(s,a) \geq \omega(s,a) \geq 0} h(\omega(s,a)) ~~\text{and}~~ \omega^*_{\geq \beta \rhat}(s,a) \defeq \argmax_{\omega(s,a) \geq \beta \rhat(s,a)} h(\omega(s,a))
\end{align*}
and the overall maximum will be either $\omega^*_{\leq \beta \rhat}(s,a)$ or $\omega^*_{\geq \beta \rhat}(s,a)$ depending on whose function value is larger (the global maximum must be one of the local maximums).

In the following, we will use the fact that $(h^\dagger)'$ is strictly decreasing due to the strict concavity of $h^\dagger$ (an affine function plus a strictly concave function).

\textbf{(1) When $e_v(s,a) > (\alpha + 1) (\log \beta + 1) + \log \rhat(s,a)$ (or equivalently $\hdagger'(\beta \rhat(s,a)) > 0$)}:

For $\beta \rhat(s,a) \geq \omega(s,a) \geq 0$, we know that $\hdagger'(\omega(s,a)) \geq \hdagger'(\beta \rhat(s,a)) = e_v(s,a) - (\alpha + 1)(\log \beta + 1) - \log \rhat(s,a) > 0$, so $\omega^*_{\leq \beta \rhat}(s,a) = \beta \rhat(s,a)$.

For $\omega(s,a) \geq \beta \rhat(s,a)$, since $(h^\dagger)'$ is strictly decreasing and $(h^\dagger)'(\beta \rhat(s,a)) > 0$, we know that $\omega^*_{\geq \beta \rhat}(s,a)$ is attained at $\hdagger'(\omega(s,a)) = e_v(s,a) - (\alpha + 1) (\log(\omega(s,a)) + 1) + \alpha \log \rhat(s,a) = 0$. Thus $\omega^*_{\geq \beta \rhat}(s,a) = \exp\left(\frac{e_v(s,a) + \alpha \log \rhat(s,a)}{1 + \alpha} - 1\right)$. 

Moreover, because 
$h(\omega^*_{\leq \beta \rhat}(s,a)) = h(\beta \rhat(s,a)) < h(\omega^*_{\geq \beta \rhat}(s,a))$, and the overall maximum is $\omega^*(s,a) = \exp\left(\frac{e_v(s,a) + \alpha \log \rhat(s,a)}{1 + \alpha} - 1\right) > \beta \rhat(s,a)$.

In this case, the maximum function value of $h$ is:
\begin{align*}
    h(\omega^*(s,a)) &= e_v(s,a) \omega^*(s,a) - \omega^*(s,a) \log(\omega^*(s,a)) - \alpha \omega^*(s,a) \log \left( \frac{\omega^*(s,a)}{\hat{r}(s,a)} \right)  - \alpha C_{f, \beta} \hat{r}(s,a) \\
    &= \omega^*(s,a) \left(e_v(s,a) - (\alpha + 1) \log (\omega^*(s,a)) + \alpha \log(\rhat(s,a))\right) - \alpha C_{f, \beta} \rhat(s,a) \\
    &= (\alpha + 1) \omega^*(s,a) - \alpha C_{f, \beta} \rhat(s,a) \\
    &= (\alpha + 1) \exp\left(\frac{e_v(s,a) + \alpha \log \rhat(s,a)}{1 + \alpha} - 1\right) - \alpha C_{f, \beta} \rhat(s,a)
\end{align*}

\textbf{(2) When $e_v(s,a) \leq (\alpha + 1) (\log \beta + 1) + \log \rhat(s,a)$ (or equivalently $\hdagger'(\beta \rhat(s,a)) \leq 0$)} :

Since $\lim_{\omega(s,a) \to 0} \hdagger'(\omega) = + \infty$ and $\hdagger'(\beta \rhat(s,a)) \leq 0$, $\omega^*_{\leq \beta \rhat}(s,a)$ is attained at $\hdagger'(\omega(s,a)) = e_v(s,a) - \log(\omega(s,a)) - 1 - \alpha (\log \beta + 1) = 0$. Thus $\omega^*_{\leq \beta \rhat}(s,a) = \exp\left(e_v(s,a) - 1 - \alpha(\log \beta + 1)\right)$.

For $\omega(s,a) \geq \beta \rhat(s,a)$, since $\hdagger'$ is strictly decreasing and $\hdagger'(\beta \rhat(s,a)) \leq 0$, so $\hdagger'(\omega(s,a)) \leq 0$ and $\omega^*_{\geq \beta \rhat}(s,a) = \beta \rhat(s,a)$.

Moreover, because $h(\omega^*_{\leq \beta \rhat}(s,a)) \geq h(\beta \rhat(s,a)) = h(\omega^*_{\geq \beta \rhat}(s,a))$, the overall maximum is $\omega^*(s,a) = \exp\left(e_v(s,a) - 1 - \alpha(\log \beta + 1)\right)$.

In this case, the maximum function value of $h$ is:
\begin{align*}
    h(\omega^*(s,a)) &= e_v(s,a) \omega^*(s,a) - \omega^*(s,a) \log(\omega^*(s,a)) - \alpha f'(\beta) \omega^*(s,a) + \alpha f'(\beta) \hat{r}(s,a) \\
    &= \omega^*(s,a) \left(e_v(s,a) - \log (\omega^*(s,a)) - \alpha (\log \beta + 1) \right) + \alpha (\log \beta + 1) \rhat(s,a) \\
    &= \omega^*(s,a) + \alpha (\log \beta + 1)\rhat(s,a) \\
    &= \exp\left(e_v(s,a) - 1 - \alpha(\log \beta + 1)\right) + \alpha (\log \beta + 1)\rhat(s,a)
\end{align*}
\end{proof}

\newpage
\section{Additional Experimental Results and Details}\label{app:details}

\subsection{Sensitivity of Hyperparameters
for Controlling Regularization Strength}
\begin{figure*}[h]
\centering
\includegraphics[width=.245\textwidth]{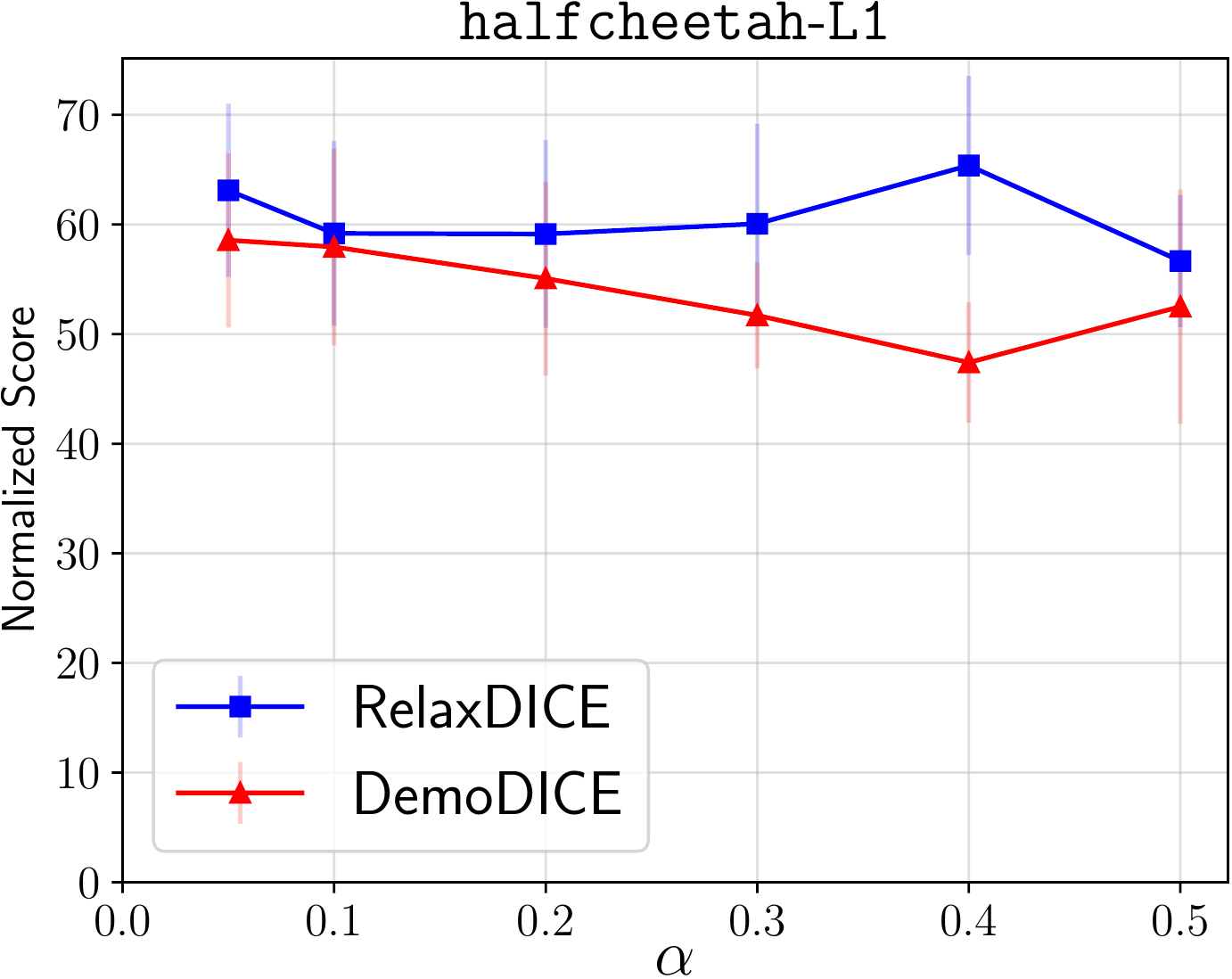}
\includegraphics[width=.245\textwidth]{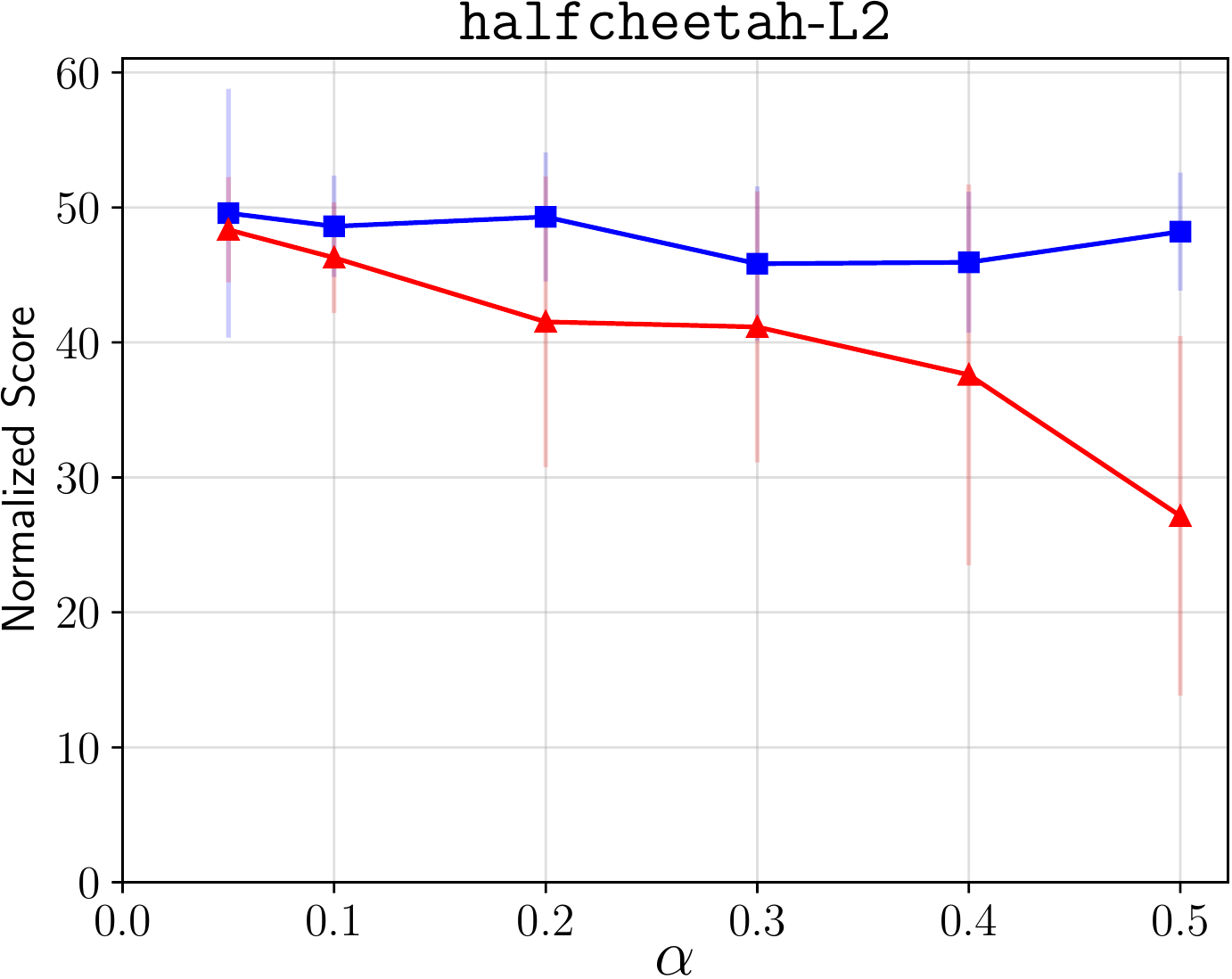}
\includegraphics[width=.245\textwidth]{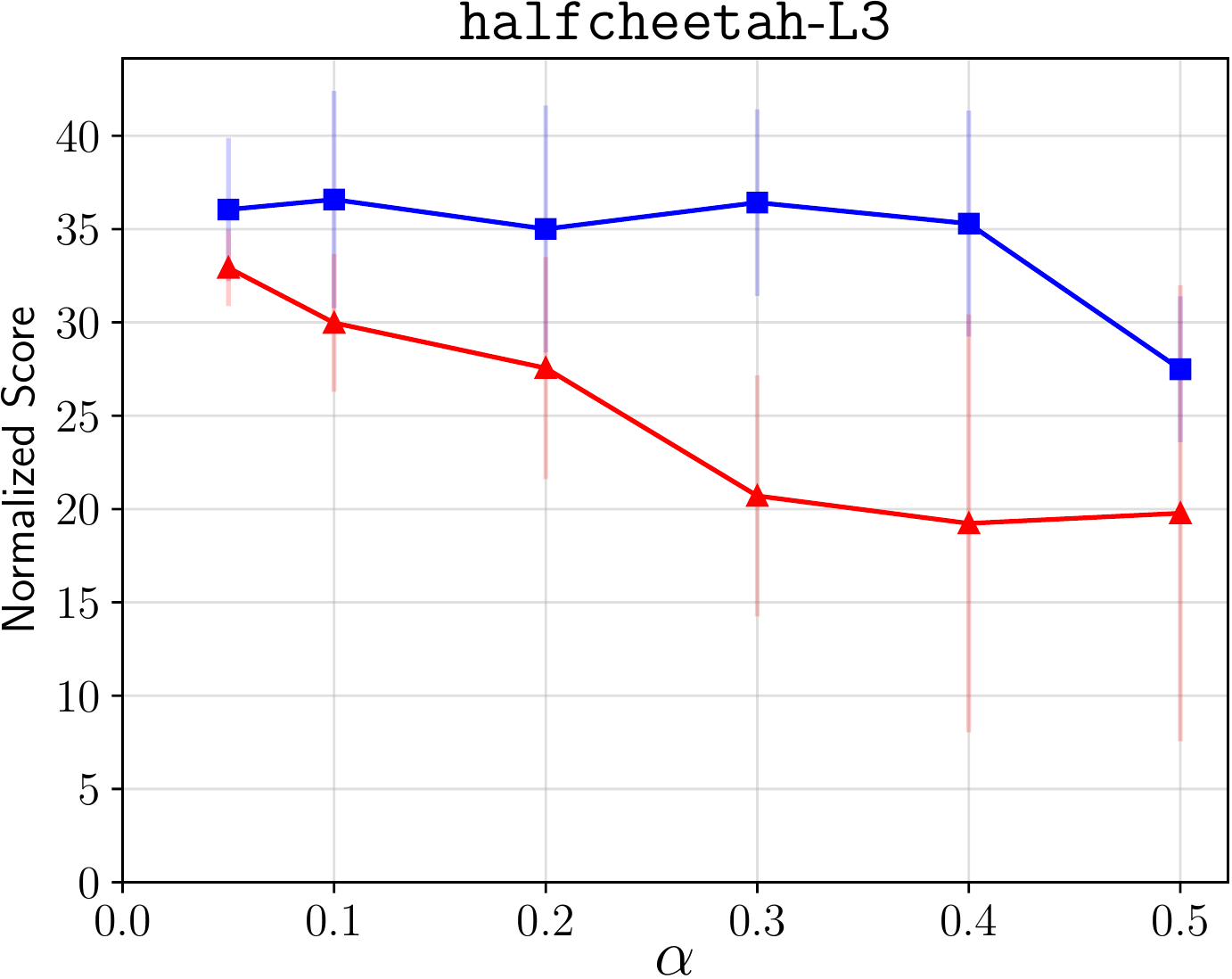}
\includegraphics[width=.245\textwidth]{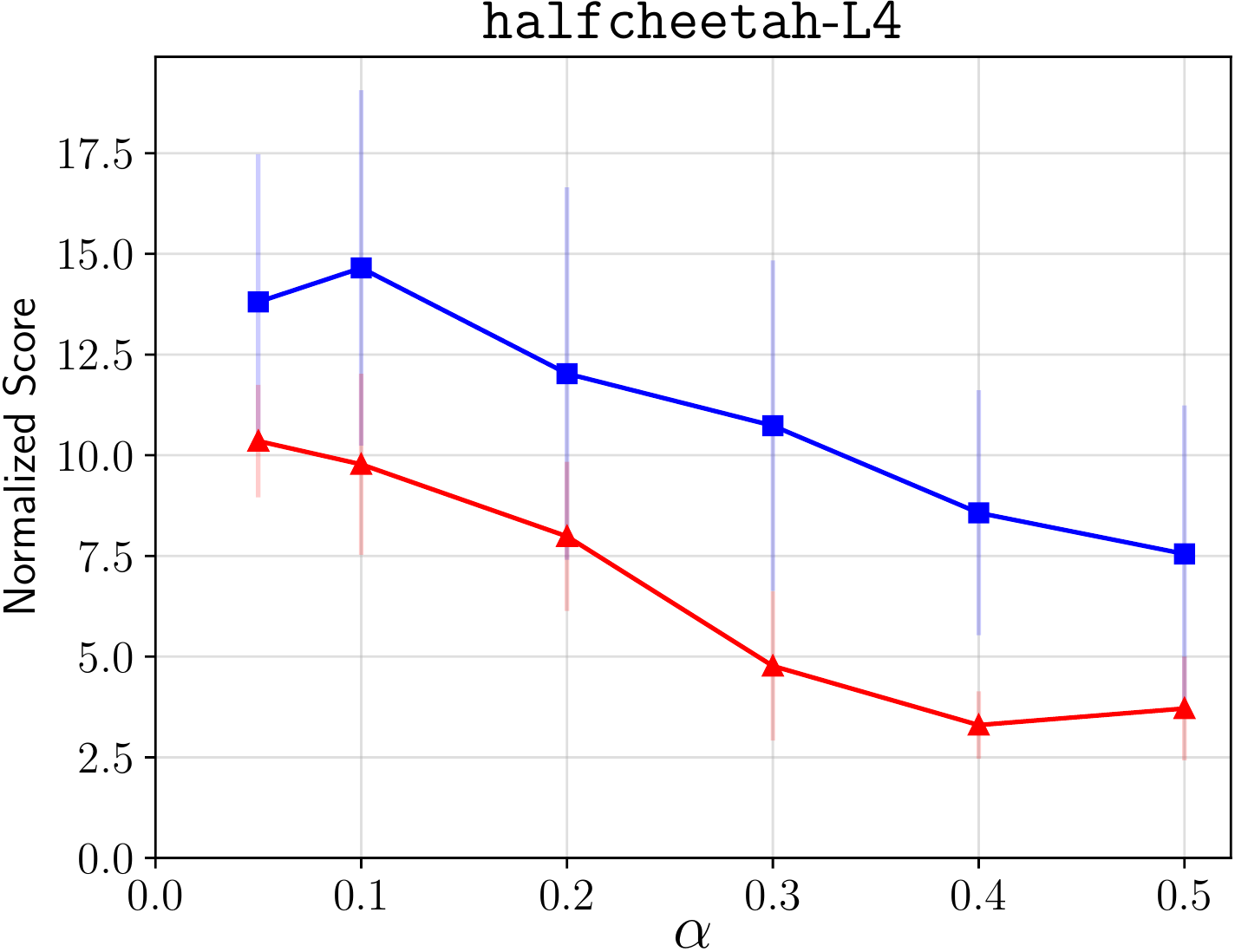}
\includegraphics[width=.245\textwidth]{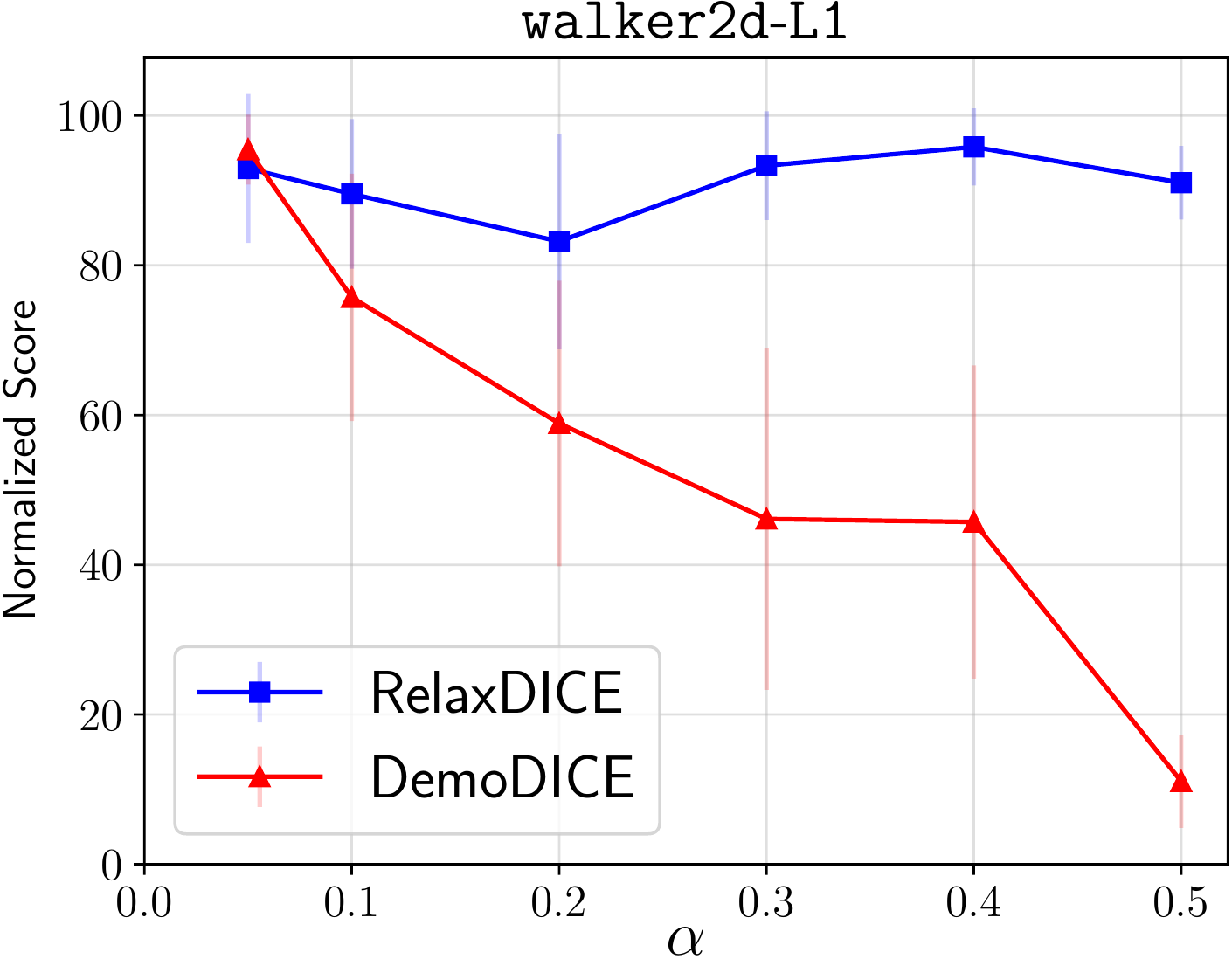}
\includegraphics[width=.245\textwidth]{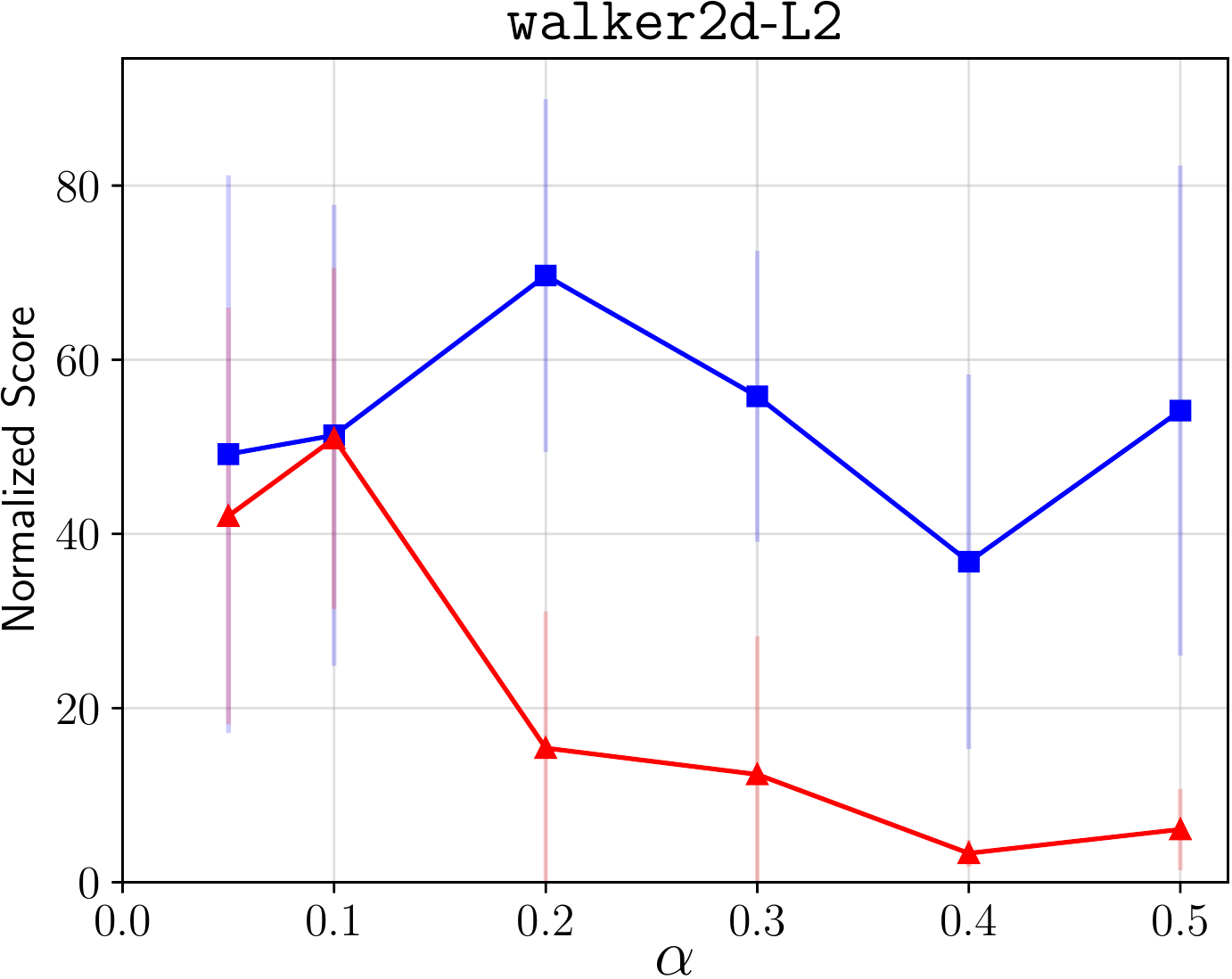}
\includegraphics[width=.245\textwidth]{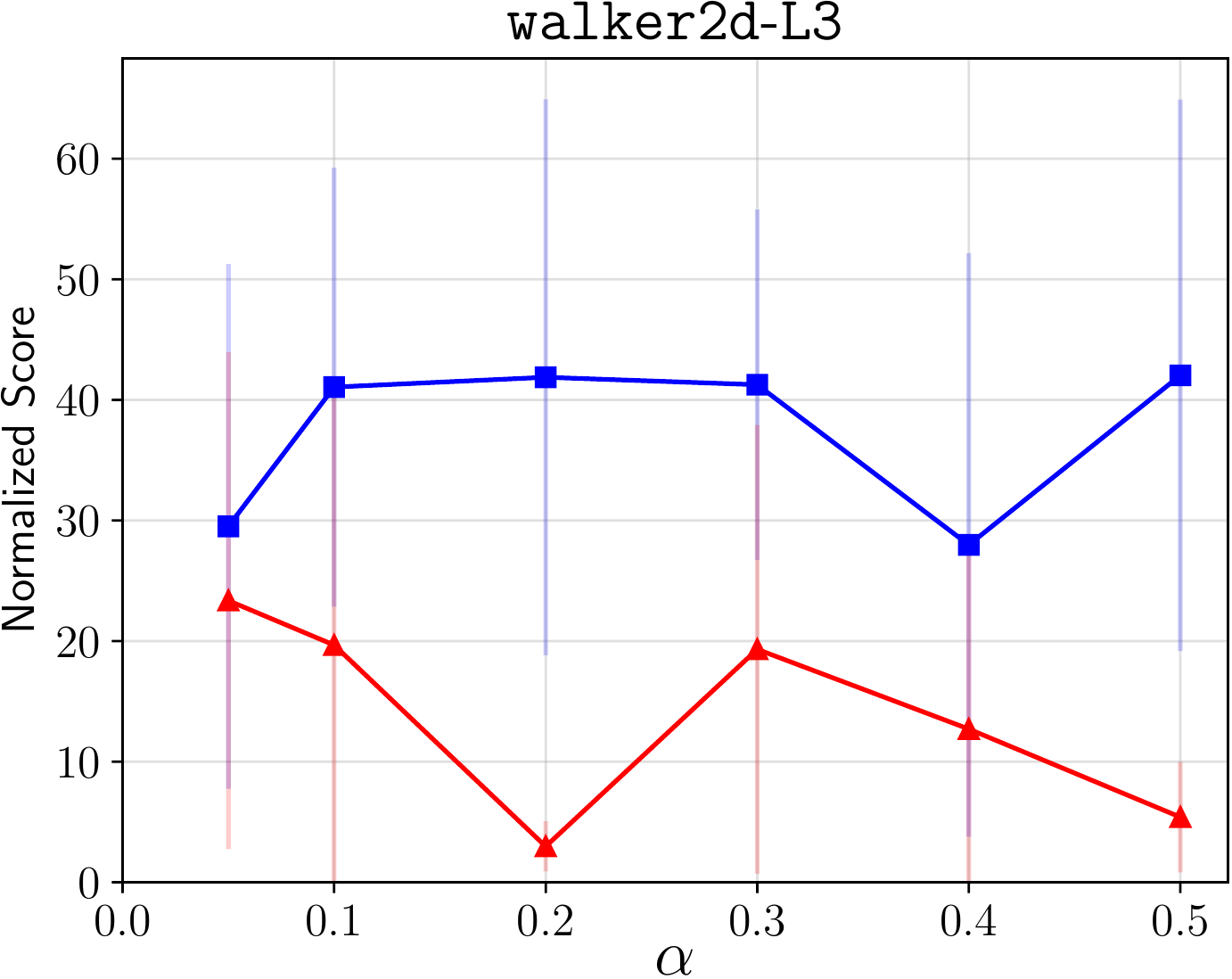}
\includegraphics[width=.245\textwidth]{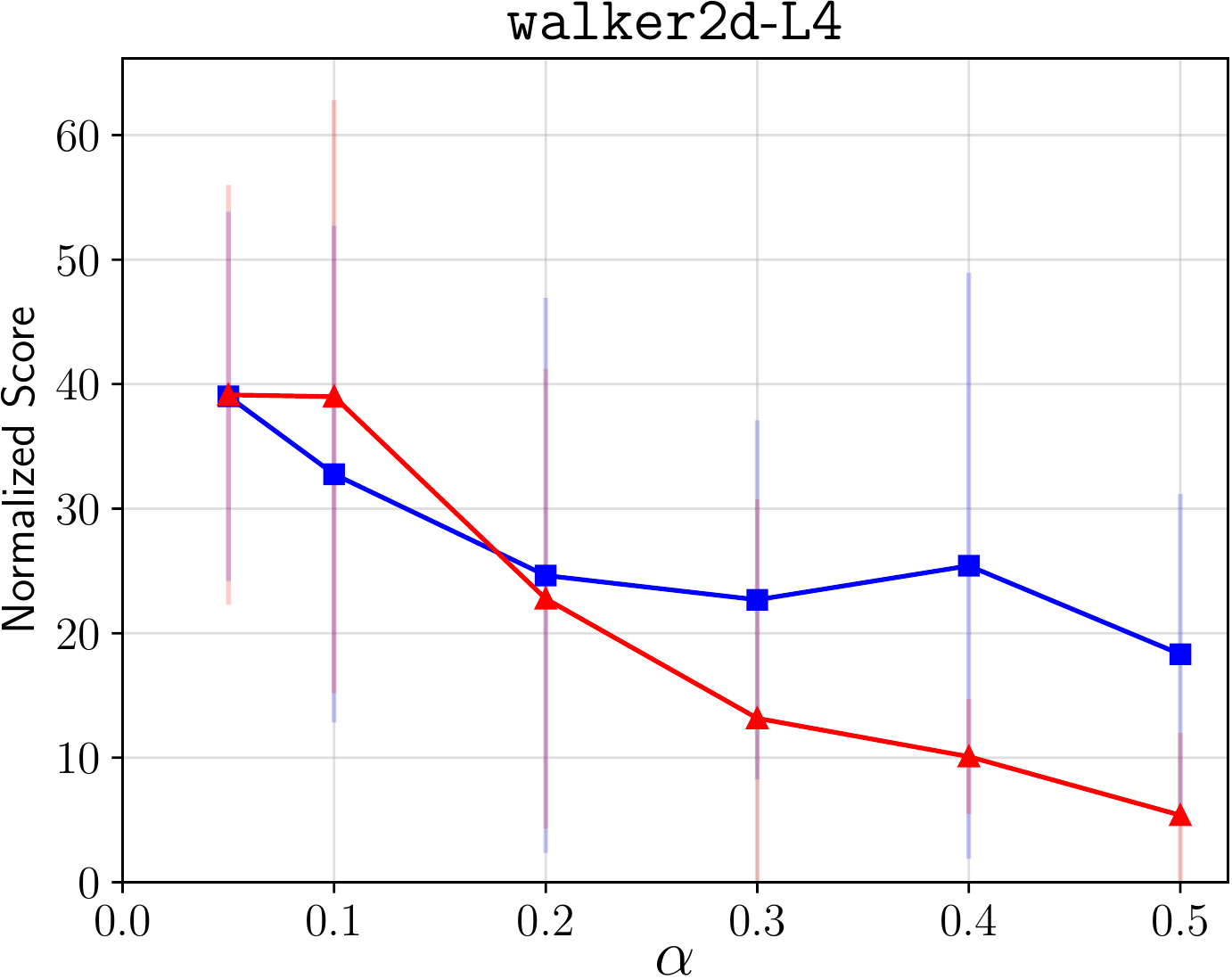}
\caption{\footnotesize Ablation study on the sensitivity of $\alpha$ for \methodshort{} and DemoDICE~\citep{kim2021demodice}. \methodshort{} is much more robust w.r.t. different values of $\alpha$ compared to DemoDICE across different data compositions.}
\label{fig:ablation}
\end{figure*}

\subsection{Task Description}
We consider offline datasets of four MuJoCo~\citep{todorov2012mujoco} locomotion environments (\texttt{hopper}, \texttt{halfcheetah}, \texttt{walker2d} and \texttt{ant}) and two Adroit robotic manipulation environments (\texttt{hammer} and \texttt{relocate}) from the standard offline RL benchmark D4RL~\citep{fu2020d4rl}. 
For each environment, we construct different settings where there is a limited amount of expert demonstrations (denoted as $\gD^E$) and a relatively large collection of suboptimal trajectories (denoted as $\gD^U$) by mixing $N^E$ transitions from expert datasets and $N^R$ transitions from extremely low-quality datasets\footnote{We use \texttt{random-v2} for four MuJoCo locomotion environments and \texttt{cloned-v0} for two Adroit robotic manipulation environments as the extremely low-quality datasets, because performing behavior cloning on the full set of these datasets has a near-zero normalized score (see Table 2 in \citep{fu2020d4rl}).}. 
We denote these settings as \texttt{L1} (Level 1), \texttt{L2} (Level 2), \texttt{L3} (Level 3) and \texttt{L4} (Level 4), where a higher level means a more challenging setting. Note that for Adroit environments, we use three settings \texttt{L1}, \texttt{L2} and \texttt{L3} instead of four due to the complexity of the high-dimensional tasks. It's also worth noting that all the considered tasks here are much more challenging than the settings in \citep{kim2021demodice}, in the sense that these settings are of much more suboptimal data composition (even tasks \texttt{L1} have lower $N^E/N^R$ ratios than the most challenging tasks in \citep{kim2021demodice}).
We summarize the details of these tasks in Table~\ref{tbl:tasks}.

\begin{table*}[ht]
\vspace{-0.2cm}
\begin{center}
\resizebox{\textwidth}{!}{\begin{tabular}{l|l|c|cc}
\toprule
& & Expert Dataset $\gD^E$ & \multicolumn{2}{c}{Suboptimal Dataset $\gD^U$} \\
\textbf{Envs} & \textbf{Tasks} & \# of transitions from \texttt{expert-v2} & \# of transitions from \texttt{expert-v2} & \# of transitions from \texttt{random-v2} \\ \midrule
& \texttt{L1}  & \texttt{1k} & \texttt{200k} & \texttt{1000k}\\
& \texttt{L2}  & \texttt{1k} & \texttt{150k} & \texttt{1000k}\\
\texttt{halfcheetah}& \texttt{L3} & \texttt{1k} & \texttt{100k} & \texttt{1000k}\\
& \texttt{L4} & \texttt{1k} & \texttt{50k} & \texttt{1000k}\\
\midrule
& \texttt{L1} & \texttt{1k} & \texttt{14k} & \texttt{22k}\\
& \texttt{L2} & \texttt{1k} & \texttt{10k} & \texttt{22k} \\
\texttt{hopper}& \texttt{L3} & \texttt{1k} & \texttt{5k} & \texttt{22k}\\
& \texttt{L4} & \texttt{1k} & \texttt{2k} & \texttt{22k}\\
\midrule
& \texttt{L1}  & \texttt{1k} & \texttt{10k} & \texttt{20k}\\
& \texttt{L2}  & \texttt{1k} & \texttt{5k} & \texttt{20k}\\
\texttt{walker2d}& \texttt{L3} & \texttt{1k} & \texttt{3k} & \texttt{20k}\\
& \texttt{L4} & \texttt{1k} & \texttt{2k} & \texttt{20k}\\
\midrule
& \texttt{L1} & \texttt{1k} & \texttt{30k} & \texttt{180k} \\
& \texttt{L2}  & \texttt{1k} & \texttt{20k} & \texttt{180k} \\
\texttt{ant}& \texttt{L3}  &\texttt{1k} & \texttt{10k} & \texttt{180k} \\
& \texttt{L4}  & \texttt{1k} & \texttt{5k} & \texttt{180k} \\
\midrule
\midrule
& & Expert Dataset $\gD^E$ & \multicolumn{2}{c}{Suboptimal Dataset $\gD^U$} \\
\textbf{Envs} & \textbf{Tasks} & \# of transitions from \texttt{expert-v0} & \# of transitions from \texttt{expert-v0} & \# of transitions from \texttt{cloned-v0} \\ \midrule
& \texttt{L1} &\texttt{2k} &\texttt{1000k} &\texttt{1000k}\\
\texttt{hammer}& \texttt{L2} &\texttt{2k} &\texttt{790k} &\texttt{1000k}\\
& \texttt{L3} &\texttt{2k} &\texttt{590k} &\texttt{1000k}\\
\midrule
& \texttt{L1} &\texttt{10k} &\texttt{1000k} &\texttt{1000k}\\
\texttt{relocate}& \texttt{L2} &\texttt{10k} &\texttt{790k} &\texttt{1000k}\\
& \texttt{L3} &\texttt{10k} &\texttt{590k} &\texttt{1000k}\\
\bottomrule
\end{tabular}}
\end{center}
\caption{Dataset statistics for four MuJoCo environments \texttt{halfcheetah, hopper, walker2d} and \texttt{ant} and two Adroit environments \texttt{hammer} and \texttt{relocate} from D4RL~\citep{fu2020d4rl}.}
\normalsize
\label{tbl:tasks}
\end{table*}

\subsection{Evaluation Protocols}
For all the methods except BCND, we run 1M training iterations (gradient steps) and we report the average performance of the last 50k (5\%) steps to capture their asymptotic performance at convergence. For BCND, we follow the implementation in Appendix E.1 in \citep{kim2021demodice}, which has predefined number of iterations according to dataset statistics and we also report the average performance of the last 50k steps for consistent evaluation.

For four MuJoCo locomotion environments \texttt{halfcheetah, hopper, walker2d} and \texttt{ant}, we follow \citep{kim2021demodice} to compute the normalized score as:
$$
\texttt{normalized\_score} = 100 \times \frac{\texttt{score} - \texttt{random\_score}}{\texttt{expert\_score} - \texttt{random\_score}}
$$
where the \texttt{expert\_score} and \texttt{random\_score} corresponds to the average return of trajectories in \texttt{expert-v2} and \texttt{random-v2} respectively.

For two Adroit environments \texttt{hammer} and \texttt{relocate}, we use the recommended reference score in D4RL\footnote{\url{https://github.com/rail-berkeley/d4rl/blob/master/d4rl/infos.py}} to compute the normalized score.

\subsection{Hyperparameters}
\subsubsection{Algorithm Hyperparameters}
We use $\gamma = 0.99$ as the discount factor of the MDP. All methods use a batch size of $256$. The hyperparameters for each of the compared algorithm are summarized below.
\begin{itemize}
    \item BC: $\eta \in \{0.0, 0.5, 1.0\}$.
    \item BC-DRC: $\eta \in \{0.0, 0.5\}$.
    \item BCND: We follow the hyperparameter configurations in Appendix E.1 in \citep{kim2021demodice}.
    \item DemoDICE\footnote{\url{https://github.com/geon-hyeong/imitation-dice}}: $\alpha = 0.05$ across all tasks, as suggested in \citep{kim2021demodice} and verified in our experiments.
    \item \methodshort{}: $\alpha = 0.2$ across all tasks. As shown in Figure~\ref{fig:ablation}, \methodshort{} can achieve a potentially better performance if using different values for different tasks. Nevertheless, we fix $\alpha = 0.2$ for all tasks to demonstrate the robustness of \methodshort{} across different tasks. Moreover, we automatically set $\beta$ to be the running average of the maximum estimated density ratio ($\rhat_\theta = \frac{c_\theta}{1 - c_\theta}$) of each minibatch.
    \item \methodshortdrc{}: We perform a grid search for $\alpha \in \{0.05, 0.1, 0.2, 0.5, 1.0\}$ and $\beta \in \{1.5, 2.0\}$. Note that here we can potentially use a larger regularization strength $\alpha$ because we employ a relaxed support regularization based on the density-ratio-corrected occupancy measure $\rhat \cdot d^U$, which has the potential for better policy improvement. $\beta$ should characterize the upper bound of the density ratio $d^E/ (\rhat_\theta \cdot d^U)$, which we expect to be close to $1$ (e.g.$1.5$ or $2$) since $\rhat_\theta \cdot d^U$ is a density-ratio-corrected occupancy measure. For full reproducibility, we summarize the used hyperparameters in Table~\ref{tbl:relaxdice-drc-hyperparameters}.
\end{itemize}
\begin{table*}[ht]
\begin{center}
\resizebox{.28\textwidth}{!}{
\begin{tabular}{l|l|cc}
\toprule
\textbf{Envs} & \textbf{Tasks} & $\alpha$ & $\beta$ \\ \midrule
& \texttt{L1} & 1.0 & 1.5 \\
& \texttt{L2} & 1.0 & 1.5 \\
\texttt{hopper}& \texttt{L3} & 0.5 & 2.0 \\
& \texttt{L4}  & 0.2 & 1.5 \\
\midrule
& \texttt{L1}  & 1.0 & 1.5 \\
& \texttt{L2}  & 0.5 & 1.5 \\
\texttt{halfcheetah}& \texttt{L3}  & 0.2 & 2.0 \\
& \texttt{L4} & 0.2 & 2.0 \\
\midrule
& \texttt{L1}  & 0.2 & 2.0 \\
& \texttt{L2}  & 0.5 & 2.0 \\
\texttt{walker2d}& \texttt{L3} & 0.1 & 1.5 \\
& \texttt{L4} & 0.05 & 2.0 \\
\midrule
& \texttt{L1}  & 0.1 & 1.5 \\
& \texttt{L2}  & 0.2 & 1.5 \\
\texttt{ant}& \texttt{L3} & 0.5 & 2.0 \\
& \texttt{L4}  & 0.5 & 2.0 \\
\midrule
\midrule
& \texttt{L1} & 0.5 & 1.5 \\
\texttt{hammer}& \texttt{L2} & 0.05 & 1.5 \\
& \texttt{L3} & 0.5 & 2.0 \\
\midrule
& \texttt{L1} & 0.5 & 1.5 \\
\texttt{relocate}& \texttt{L2} & 0.5 & 2.0 \\
& \texttt{L3} & 0.05 & 1.5 \\
\bottomrule
\end{tabular}}
\end{center}
\caption{Hypeparameters used for \methodshortdrc{}.}
\normalsize
\label{tbl:relaxdice-drc-hyperparameters}
\end{table*}

\subsubsection{Implementation Details}
We follow all the other hyperparameters from \cite{kim2021demodice} listed as follows:
\begin{itemize}
    \item Policy network $\pi_\psi$ (for BC, BC-DRC, BCND, DemoDICE, \methodshort{} and \methodshortdrc{}): three-layer MLP with $256$ hidden units, learning rate $3 \times 10^{-5}$.
    \item Lagrange multiplier network $v_\phi$ (for DemoDICE, \methodshort{} and \methodshortdrc{}): three-layer MLP with $256$ hidden units, learning rate $3 \times 10^{-4}$, gradient penalty coefficient $1 \times 10^{-4}$.
    \item Classifier network $c_\theta$ (for DemoDICE, \methodshort{} and \methodshortdrc{}): three-layer MLP with $256$ hidden units, learning rate $3 \times 10^{-4}$, gradient penalty coefficient $10$.
\end{itemize}

\subsection{Computation Resources}

We train \methodshort{} and \methodshortdrc{} on a single NVIDIA GeForce RTX 2080 Ti with 5 random seeds for at most 3 hours for all the tasks.

\subsection{Dataset License}

All datasets in our experiments are from the open-sourced D4RL\footnote{\url{https://github.com/rail-berkeley/d4rl}} benchmark~\citep{fu2020d4rl}. All datasets there are licensed under the Creative Commons Attribution 4.0 License (CC BY).

\end{document}